\documentclass[10pt,twocolumn,letterpaper]{article}

\usepackage{wacv}
\usepackage{times}
\usepackage{epsfig}
\usepackage{graphicx}
\usepackage{amsmath}
\usepackage{amssymb}
\usepackage{multirow}
\usepackage{amsthm}
\usepackage{graphicx}
\usepackage{amsmath,amssymb} % define this before the line numbering.
\usepackage{color}

\usepackage{graphicx} % Required for including pictures

\usepackage{float} % Allows putting an [H] in \begin{figure} to specify the exact location of the figure
\usepackage{cite}

\usepackage{bm}
\usepackage{amsmath}
\usepackage{url}
\usepackage{algorithm}
\usepackage{xcolor}
\usepackage{algpseudocode}
\DeclareMathOperator{\spn}{span}
\DeclareMathOperator{\prox}{\mathbf{prox}}

\wacvfinalcopy % *** Uncomment this line for the final submission

\def\wacvPaperID{376} % *** Enter the wacv Paper ID here

\ifwacvfinal\fi
\begin{document}

%%%%%%%%% TITLE
\title{Understanding Kernel Size in Blind Deconvolution}

% Authors at the same institution
%\author{First Author \hspace{2cm} Second Author \\
%Institution1\\
%{\tt\small firstauthor@i1.org}
%}
% Authors at different institutions
\author{Li Si-Yao\\
Beijing Normal University\\
{\tt\small lisiyao@mail.bnu.edu.cn}
\and
Dongwei Ren \\
Tianjin University\\
{\tt\small rendongweihit@gmail.com}
\and
Qian Yin* \\
Beijing Normal University\\
{\tt\small yinqian@bnu.edu.cn}
}

\maketitle
\ifwacvfinal\thispagestyle{empty}\fi

%%%%%%%%% ABSTRACT
\begin{abstract}
Most blind deconvolution methods usually pre-define a large kernel size to guarantee the support domain. Blur kernel estimation error is likely to be introduced, yielding severe artifacts in deblurring results. In this paper, we first theoretically and experimentally analyze the mechanism to estimation error in oversized kernel, and show that it holds even on blurry images without noises. Then to suppress this adverse effect, we propose a low rank-based regularization on blur kernel to exploit the structural information in degraded kernels, by which larger-kernel effect can be effectively suppressed. And we propose an efficient optimization algorithm to solve it. 
Experimental results on benchmark datasets show that the proposed method is comparable with the state-of-the-arts by accordingly setting proper kernel size, and performs much better in handling larger-size kernels quantitatively and qualitatively. The deblurring results on real-world blurry images further validate the effectiveness of the proposed method. 
\end{abstract}

%%%%%%%%% BODY TEXT
\section{Introduction}
Blind deconvolution is a fundamental problem in low level vision, and is always drawing research attentions \cite{perrone2014total,pan2016robust,krishnan2011blind,krishnan2009fast,pan2016blind}.
Given a blurry image $\mathbf y$, blind deconvolution aims to recover a clear version $\bf x$, in which it is crucial to first estimate blur kernel $\bf k$ successfully. 
Formally, the degradation of image blur is modeled as
\begin{equation}\label{eq1}
  \mathbf y= \mathbf x \otimes \mathbf k+ \mathbf n,
\end{equation}
where $\mathbf x$ and $\mathbf y$ are with size $M\times N$, $\mathbf k$ is with size $L\times K$,
%(a.k.a. Point Spread Function (PSF)), 
$\otimes$ is the 2D convolution operator and $\mathbf n$ is usually assumed as random Gaussian noises.
Blind deconvolution needs to jointly estimate blur kernel $\bf k$ and recover clear image $\bf x$.

The most successful blind deconvolution methods are based on the maximum-a-posterior (MAP) framework.
MAP tries to jointly estimate $\mathbf k$ and $\mathbf x$ by maximizing the posterior $p(\mathbf k, \mathbf x|\mathbf y)$, which can be further reformulated as an optimization on regularized least squares~\cite{chan1998total},
\begin{equation}\label{eq2}
  \mathbf{\hat x}, \mathbf{\hat k}=\arg\min_{\mathbf x, \mathbf k}{\left(\|
  \mathbf x\otimes \mathbf k-\mathbf y\|^2+\lambda g\left(\mathbf x\right)+\sigma h\left(\mathbf k\right)\right)}
\end{equation}
where $g$ and $h$ are prior functions designed to prefer a sharp image and an ideal kernel, respectively.
It is not trivial to solve the optimization problem in Eqn. \eqref{eq2}, and instead it is usually addressed as alternate steps,
\begin{equation}\label{eq3}
  \mathbf {\hat{x}}^{(i+1)} = \arg\min_{\mathbf x}{\left(\|\mathbf x \otimes \mathbf {\hat{k}}^{(i)}-\mathbf y\|^2+\lambda g\left(\mathbf x\right)\right)}
\end{equation}
and
\begin{equation}\label{eq4}
  \mathbf{\hat{k}}^{(i+1)} = \arg\min_{\mathbf{k}}{\left(\|\mathbf{\hat{x}}^{(i+1)}\otimes \mathbf k- \mathbf y\|^2+\sigma h\left(\mathbf k\right)\right)}.
\end{equation}
In the most blind deconvolution methods, kernel size $(L, K)$ is hyper-parameters that should be manually set.
An ideal choice is the ground truth size to constrain the support domain, which however is not available in practical applications, requiring hand-crafted tuning. 

On one hand, a smaller kernel size than ground truth cannot provide enough support domain for estimated blur kernel.
%and thus the estimated kernel is very inaccurate.
%, yielding failures to recover clear images.
Therefore, kernel size in the existing methods is usually pre-defined as a large value to guarantee support domain.
%
%Xu and Jia~\cite{xu2010two} proposed an iterative support domain detector based on the differences of elements of $\hat k$.
%
%Behind those constraints, however, mechanisms are rarely mentioned; especially, most works treated regularization $h$ as an accessory of the success of MAP methods.
%
%Only two constraints of $k$ are fully researched: (i) the sum of $k$ equals 1 and (ii) the elements of $k$ are no smaller than 0.
%With respect to (i), Kundur~\etal~\cite{kundur1996blind} indicates that the mean of a clear image is preserved in the blurry; Levin~\etal~\cite{levin2009understanding} emphasizes that the normalization prevents scaling conflicts. For (ii), the PSF is naturally greater than zero because the contribution of light is positive. These two constraints are realized as projections after k-step~\cite{krishnan2011blind, zuo2015discriminative} and are proved to play an essential role on preventing trivial (delta) solutions in alternate optimizations~\cite{perrone2014total}.
%
%
\begin{figure*}[t]
  \centering
  \setlength{\tabcolsep}{2pt}
  % Requires \usepackage{graphicx}
  \small
  \begin{tabular}{ccccc}

   \multirow{1}{*}[50pt]{\includegraphics[width=.3\linewidth]{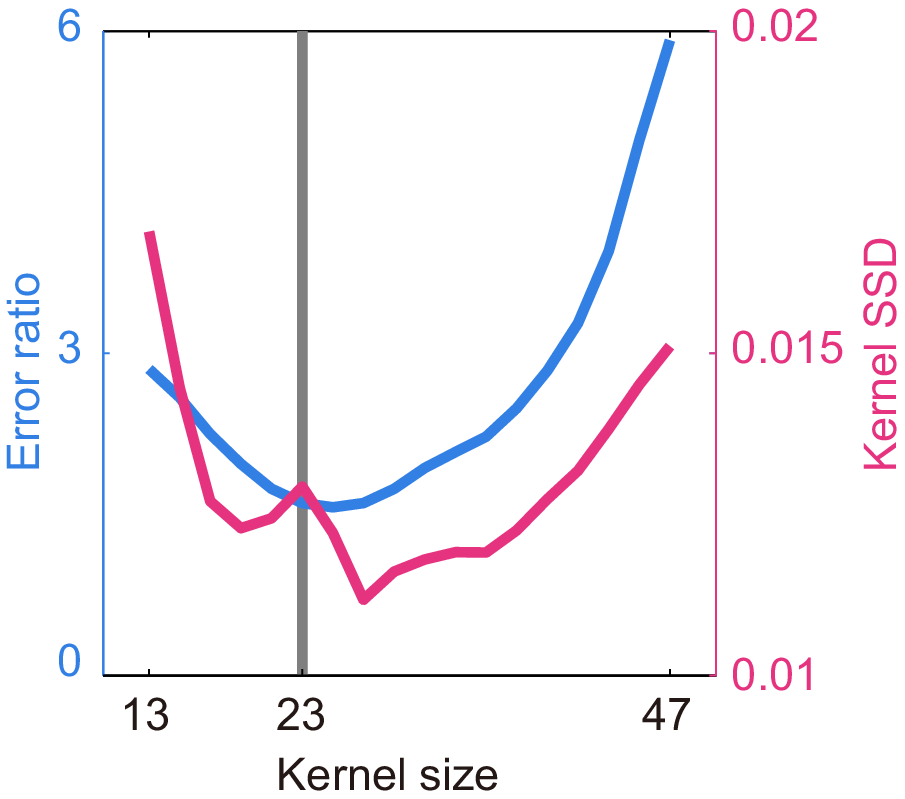}} &
 \includegraphics[width=.155\linewidth]{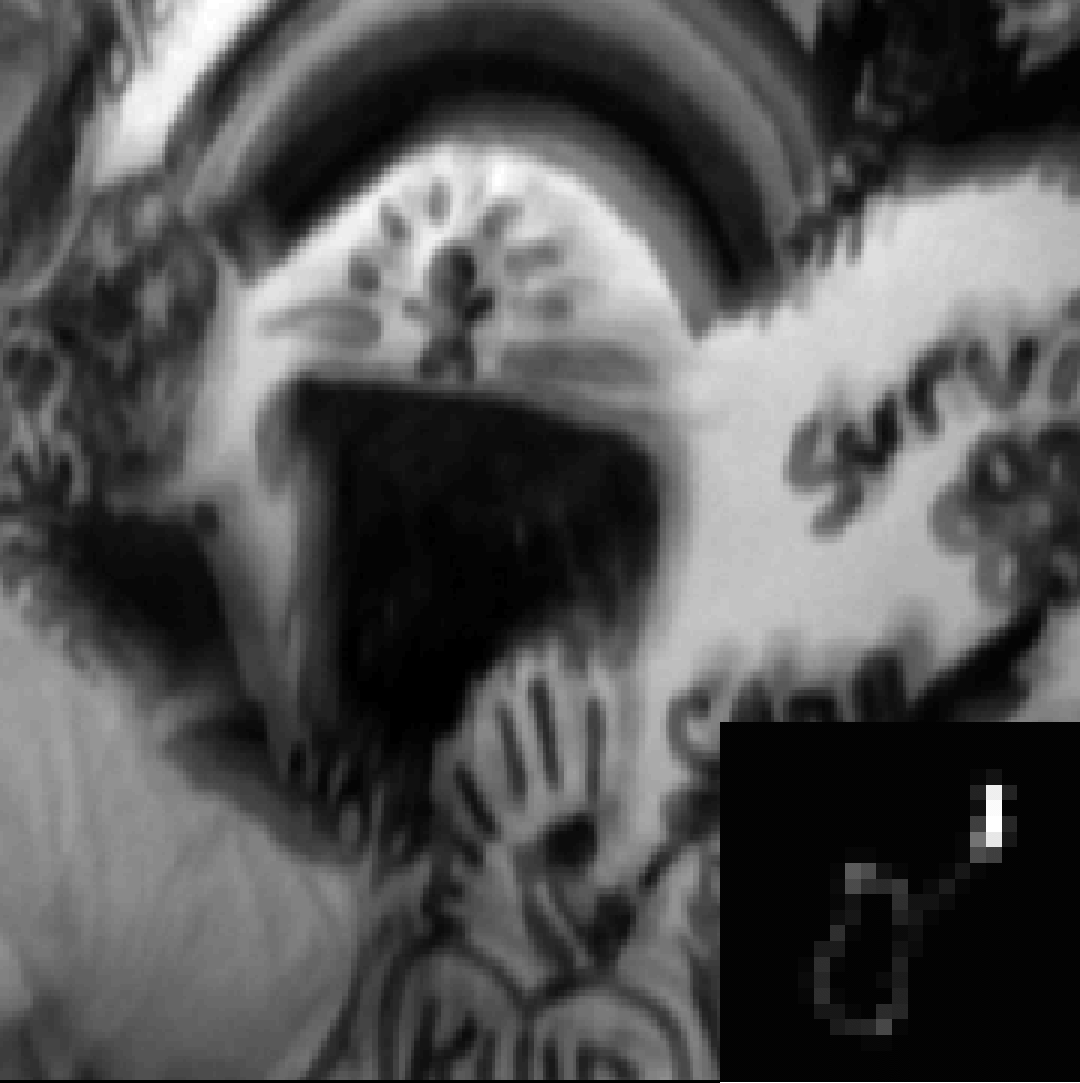} &
\includegraphics[width=.155\linewidth]{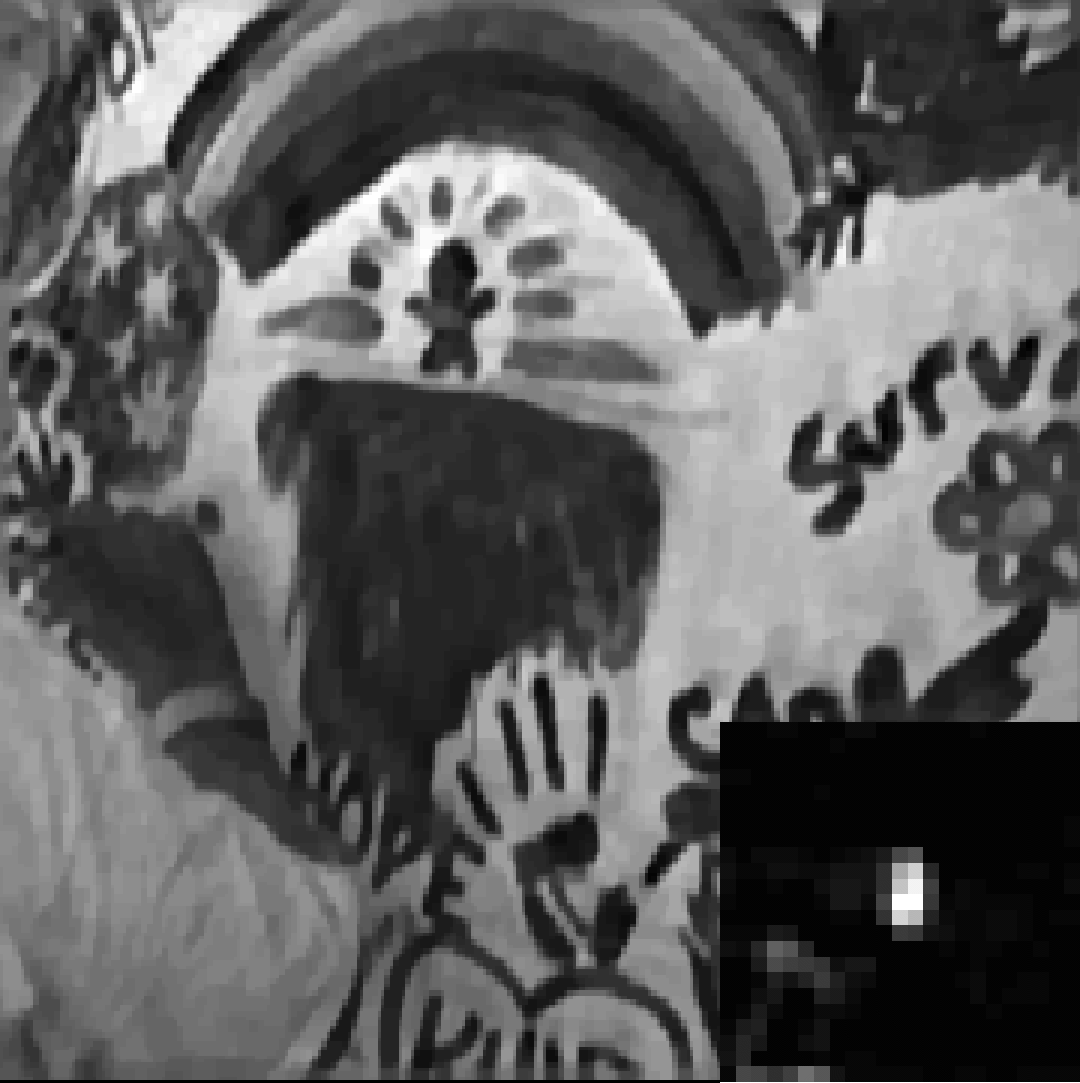} &
\includegraphics[width=.155\linewidth]{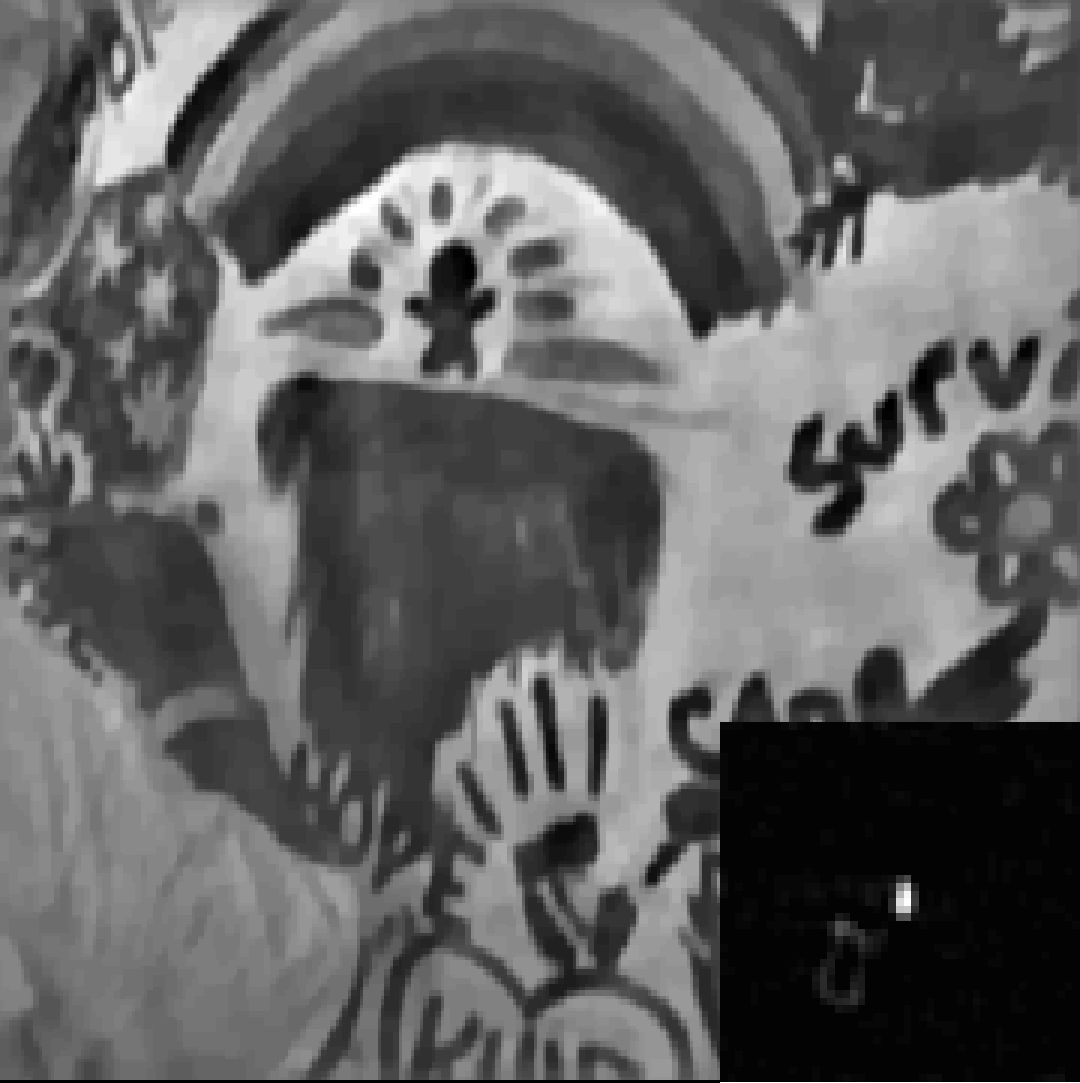} &
\includegraphics[width=.155\linewidth]{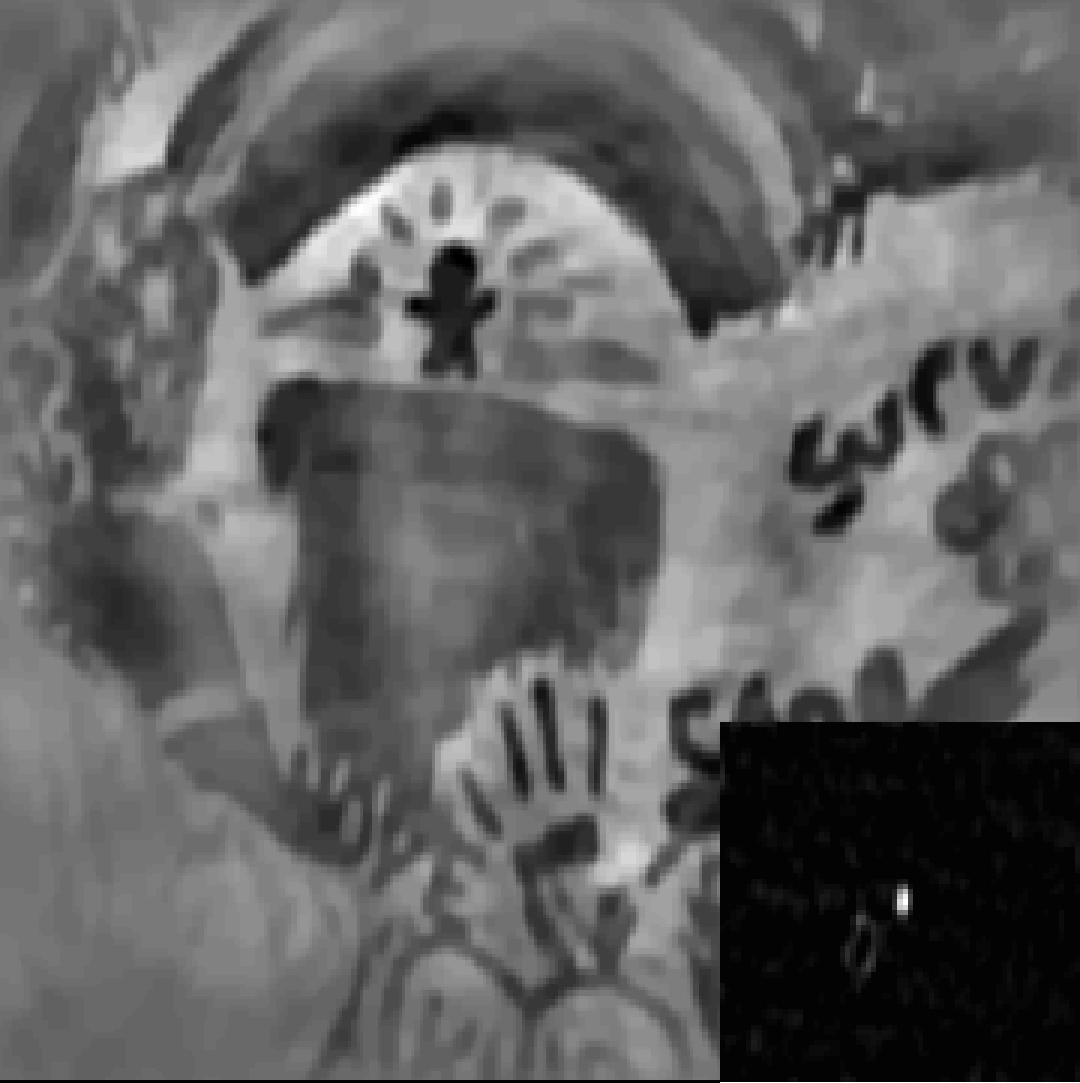} \vspace{-2pt} \\

  &\includegraphics[width=.155\linewidth]{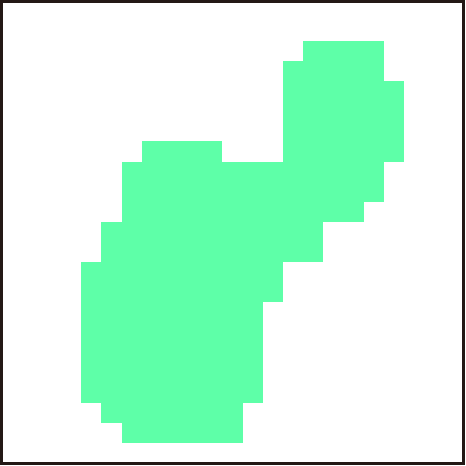} &
\includegraphics[width=.155\linewidth, trim={86.5pt 87.5pt 18pt 16pt}, clip]{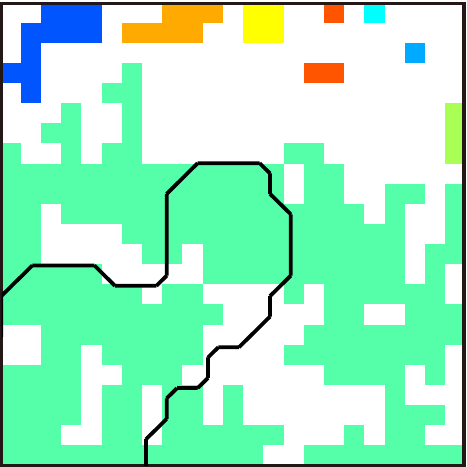} &
\includegraphics[width=.155\linewidth]{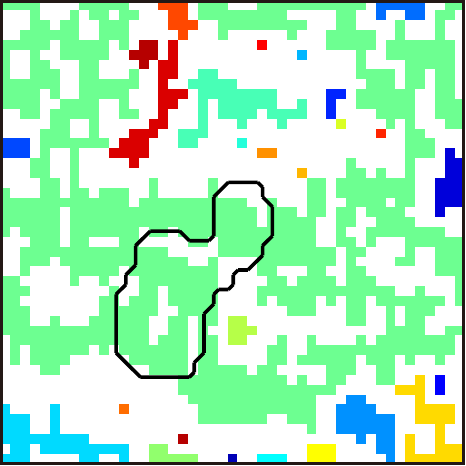} &
\includegraphics[width=.155\linewidth]{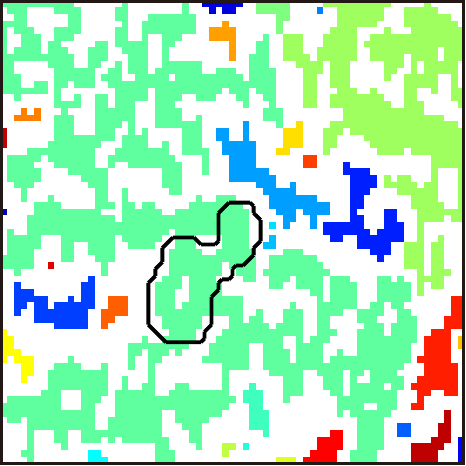} \vspace{-2pt} \\

& truth size = 23 & size = 23, err = 1.9 & size = 47, err = 5.9& size = 69, err = 69.9 \vspace{2pt}\\

(a) & (b) & (c) & (d) & (e) \vspace{10pt} \\
   
\end{tabular}\\
  \caption{Large kernels produce inferior results. (a) Numerical errors with kernel size.
  %, including error ratios (err) of deblurred images and summed square difference (SSD) of estimated kernels. 
  (b) Blurry image and ground truth kernel. (c-e) Deblurred results. In the first row of (c-e) are restored images and corresponding estimated kernels; in the second row are support domains ($k>0$), where adjacent positive pixels are colored identically and zeros are white. 
  In this experiment, we omitted regularization $h$; %to expose the large-kernel effect clearly
  hence, k-step equals to a bare least squares optimization.
% \begin{equation}\label{eq5}
%   \hat{k}^{(i+1)} = \arg\min_k{\|\hat{x}^{(i+1)}*k-y\|^2}.
% \end{equation}
We also avoided using multi-scaling scheme and threshold in this experiment. Parameters that performed well on the truth size were kept identical for larger sizes during the experiment.}\label{fig1}
\end{figure*}
On the other hand, as shown in Figure~\ref{fig1}, oversized kernels are very likely to introduce estimation errors, and hence lead to unreasonable results.
Hereby, we name this phenomenon \textit{larger-kernel effect}. 
%whose amounts are proportional to kernel size.
This interesting fact was first mentioned by Fergus~\etal~\cite{fergus2006removing}.
Then Cho and Lee~\cite{cho2016convergence} showed a similar result that the residual cost of (\ref{eq2}) increases with over-estimated kernel size.
However, such annoying phenomenon was not well analyzed and studied yet.
Note that most MAP-based blind deconvolution algorithms adopt the trial-and-error strategy to tune kernel size, so the larger-kernel effect is a very common problem. 
%
%Hence, it is significant to figure out the mechanism of larger-kernel effect.

In this paper, we first explore the mechanism of larger-kernel effect and then propose a novel low rank-based regularization to relieve this adverse effect.
%
% In Section 2, we demonstrate why kernel size affects deblurred images quantitatively.
%
Theoretically, we analyze the mechanism to introduce kernel estimation error in oversized kernel size. 
Specifically, we reformulate convolution of (\ref{eq3}) and (\ref{eq4}) to affine transformations and analyze their properties on kernel size. 
%
%The over-estimated kernel size amplifies noise in the projected kernel.
We show that for $\mathbf x$ in sparse distributions, this larger-kernel effect remains with probability one. 
We also conduct simulation experiments to show that kernel error is expected to increase with kernel size even without noise $\mathbf n$.
Furthermore, we attempt to find out a proper regularization to suppress noise in large kernels. 
By exploiting the low rank property of blur kernels, we propose a low-rank regularization to reduce noises in $\mathbf{\hat k}$, suppressing larger-kernel effect.
Experimental results on both synthetic and real blurry images validate the effectiveness of the proposed method, and show its robustness to against over-estimated kernel size. Our contributions are two-folds:
\begin{itemize}
\item
  We give a thorough analysis to mechanism of the phenomenon that over-estimated kernel size yields inferior results in blind deconvolution, on which little research attention has been paid.
\item
 We propose a low rank-based regularization to effectively suppress \textit{larger-kernel effect} along with efficient optimization algorithm, and performs favorably on oversized blur kernel than state-of-the-arts. 

\end{itemize}

\section{Larger-kernel effect}
%If an input kernel size is greater than ground truth, 
%the estimated kernel is wished to be equivalent to a zero-padded
%truth kernel. However, This effect yields image deblurring failures.
% impurities emerge in areas that is wished to be zero.
% In this section, we first describe such phenomenon in detail. Then, we explore the mechanism and the probability it happens. Finally, we provide a quantification of error increment by simulations.
In this section, we describe the \textit{larger-kernel effect} in detail and provide a mathematical explanation.

\subsection{Phenomenon}
%To demonstrate this effect clearly, we tested our implementation described in Section 4 (but set $\mu$ and $\sigma$ zero) using one-half to three times the truth size iteratively. 
In Figure ~\ref{fig1}(b-c), it has shown that the larger the kernel size would lead to more inferior deblurring results, since the estimated blur kernel with larger support domain is very likely to introduce noises and estimation errors. 
Figure ~\ref{fig1}(a) shows both the error ratio (err)~\cite{levin2009understanding} of restored images and the Summed Squared Difference (SSD) of estimated kernels reach the lowest at the truth size and increase afterwards. 
%Error ratio is defined as $$\frac{\text{Image SSD deblurred by estimated kernel}}{\text{Image SSD deblurred by ground-truth kernel}}.$$

\subsection{Mechanism}
To analyze the source of larger-kernel effect, we firstly introduce an interesting fact that we call \textit{inflating effect}.

\newtheorem{thm}{Claim}
\begin{thm}(Inflating Effect)
Let $ A = [\bm{v_1} \ldots \bm{v_n}]$, where $\bm{v_i} \in \mathbb{R}^m$ $(m \ge n + 1)$.
Let $ B = [\bm{w_1} \: A \: \bm{w_2}]$, where $ \bm{w_1}, \bm{w_2} \in \mathbb{R}^m $ and
$\mathrm{rank}(B)>\mathrm{rank}(A)$. Given an m-D random vector $\bm{b}$ whose elements
are i.i.d. with the continuous probability density function p, for $ \bm{u} \in \mathbb{R}^m$
$$ \Pr\big(\inf\{\|B\bm{u}-\bm{b}\|^2\}<\inf\{\|A\bm{u}-\bm{b}\|^2\}\big) = 1.$$
\end{thm}
\begin{proof}
\begin{equation*}
\begin{aligned}
   &&&\Pr\big(\inf\{\|B\bm{u}-\bm{b}\|^2\}\ge \inf\{\|A\bm{u}-\bm{b}\|^2\}\big) \\
 &=&&\Pr\big(\bm{b}\in \mathbb{R}^m \setminus (\spn\{B\}\setminus \spn\{A\})\big) \\
   &=&&\int_\Omega {d {p(\bm{b})}}
\end{aligned}
\end{equation*}
where $\Omega = \mathbb{R}^m \setminus (\spn\{B\}\setminus \spn\{A\})$.

For $\mathrm{rank}(B)>\mathrm{rank}(A)$, we have $\mathrm{dim}(\spn\{B\}\setminus \spn\{A\}) > 0$. Hence, the Lebesgue measure of
$\Omega$ is zero, and the probability is zero.
\end{proof}

Claim 1 shows that padding linear independent columns to a thin matrix leads to a different least squares solution with lower residue squared cost.

The convolution part in (\ref{eq1}) is equivalent to linear transforms:
\begin{equation}\label{eq6}
  \bm y = \mathbf T_{\mathbf{k}}\bm{x} + \bm n = \mathbf T_{\mathbf{x}}\bm{k} + \bm n.
\end{equation}
where italic letters $\bm y, \bm x$, $ \bm k $ and $\bm n $ represent column-wise expanded vectors of 2D $\mathbf y, \mathbf x$, $ \mathbf k $ and $\mathbf n $, respectively; $\mathbf T_{\mathbf k}\in\mathbb{R}^{MN\times MN}$ and $\mathbf{T_{\mathbf x}}\in \mathbb{R}^{MN\times LK}$ are blocked banded Toeplitz matrices~\cite{andrews1977digital, gray2006toeplitz}; $L$ and $K$ are required to be odd.

We attribute the larger-kernel effect to either substep  \eqref{eq3} or \eqref{eq4}. On one hand, $\mathbf T_{\mathbf k}$ remains identical when $L$ and $K$ increase by wrapping a layer of zeros around $\mathbf k$ and the result of x-step keeps the same. Hence, x-step should not be blamed as the source of the larger-kernel effect. On the other hand, when $\mathbf k$ is larger, $\mathbf T_{\mathbf x}$ will become inflated for the same $\mathbf x$. In 1D cases, where $N = K = 1$, assume $L = 2l + 1$, then
\begin{equation}\label{eq_Tx}
\begin{aligned}
  &\mathbf T_{\bm{x}}(L) \\= &\left[\begin{matrix}
x_{l+1}&   \cdots&  x_2&    x_1&    0&      \cdots& 0       \\
\vdots &    &       \vdots& x_2&    x_1&    \ddots& \vdots  \\
x_{M-1}&    &       \vdots& \vdots& x_2&    \ddots& 0       \\
x_M&        \ddots& \vdots& \vdots& \vdots& \ddots& x_1       \\
0&        \ddots& x_{M-1}& \vdots& \vdots& & x_2       \\
\vdots&       \ddots& x_M& x_{M-1}& \vdots& & \vdots       \\
0&        \cdots& 0& x_M& x_{M-1}&  \cdots& x_{M-l}
\end{matrix}\right] %\\
%&=\left[ J^{(-l)}\bm{x} \;  \cdots \;J^{(-1)}\bm{x} \;\; \bm{x} \;\; J^{(1)}\bm{x} \;
%\cdots \; J^{(l)}\bm{x}\right].
\end{aligned}
\end{equation}

During blind deconvolution iterations, for identical values of $\bm{\hat x}^{(i)}$, a larger $L$ introduces more columns onto both sizes of $\mathbf T_{\bm{\hat x}^{(i)}}$ and results in different
solutions. To illustrate this point, we tested a 1D version of blind deconvolution without kernel regularization and took different values of $L$ (truth and double and four times the truth size) for the 50th
 k-step optimization after 49 truth-size iterations (see Figure ~\ref{fig2}). Figure ~\ref{fig2}(a-c) show that the optimal solutions in different sizes differ slightly on the main body
 that lies within the ground truth size (colored in red), but greatly outside this range (colored in green) where zeros are expected. Figure ~\ref{fig2}(d-f) compare
   ground truth to estimated kernels in (a-c) after non-negativity and sum-to-one projections.
 Larger sizes yield more  positive noises; hence, they lower the
   weight of the main body after projections and change the outlook of estimated kernel.

\begin{figure}[t]
\begin{center}
%\fbox{\rule{0pt}{2in} \rule{0.9\linewidth}{0pt}}
   \setlength{\tabcolsep}{5pt}
  % Requires \usepackage{graphicx}
  \small
  \begin{tabular}{cccc}

 &  \includegraphics[width=.4\linewidth]{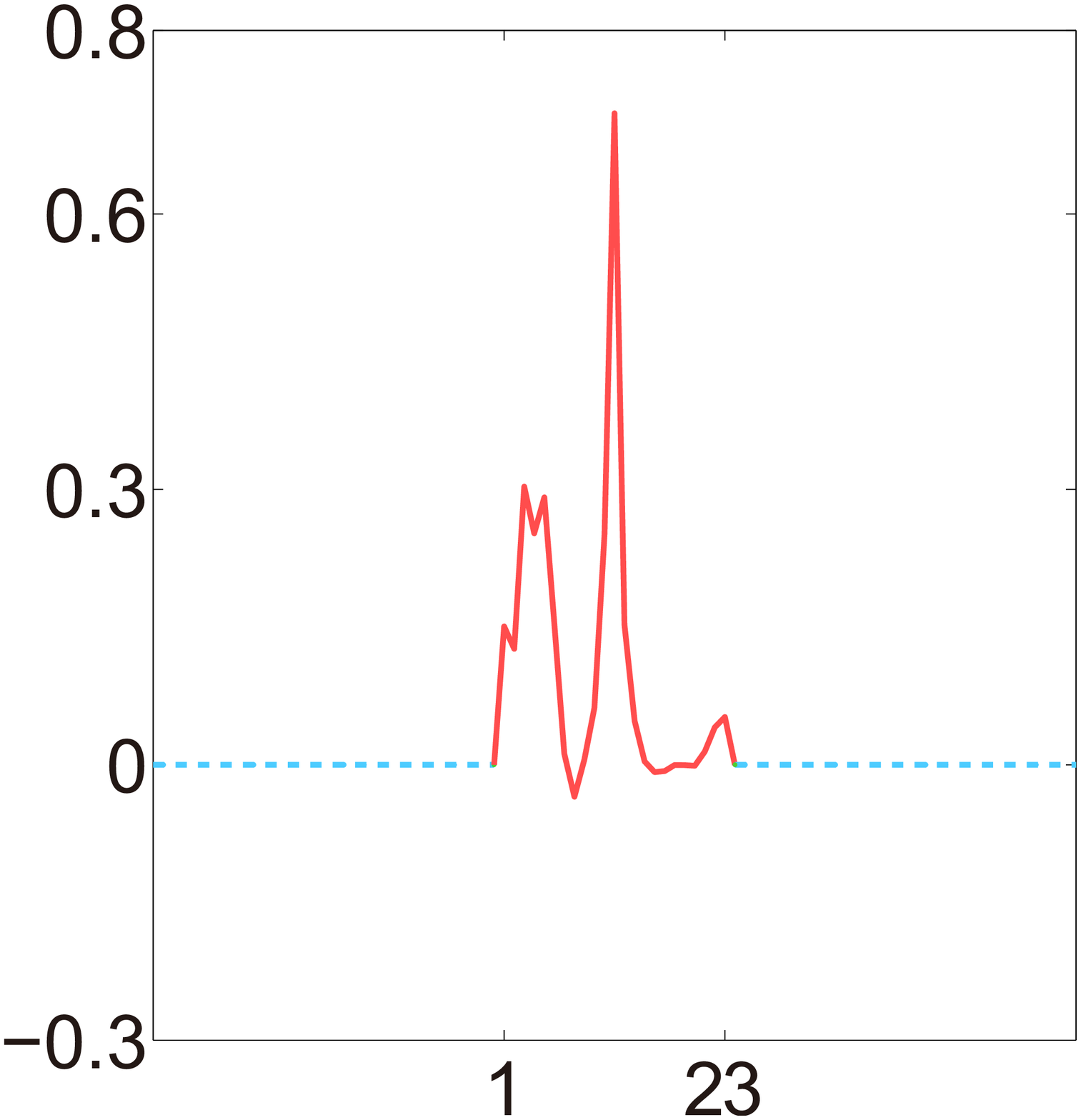} & 
 \includegraphics[width=.4\linewidth]{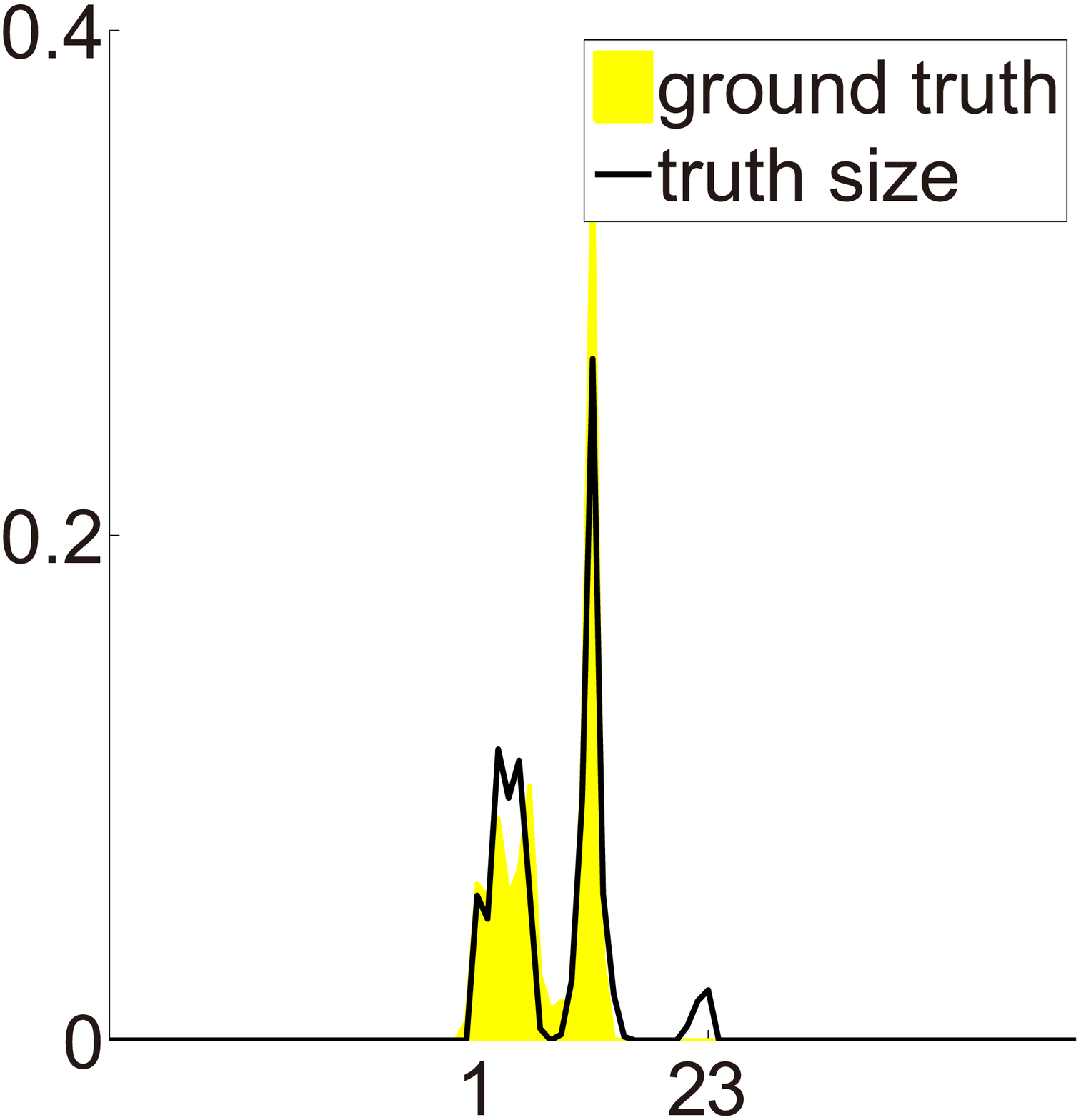} \vspace{0pt} \\
 & (a) & (d)\\

\rotatebox{90}{Estimated kernels} & \includegraphics[width=.4\linewidth]{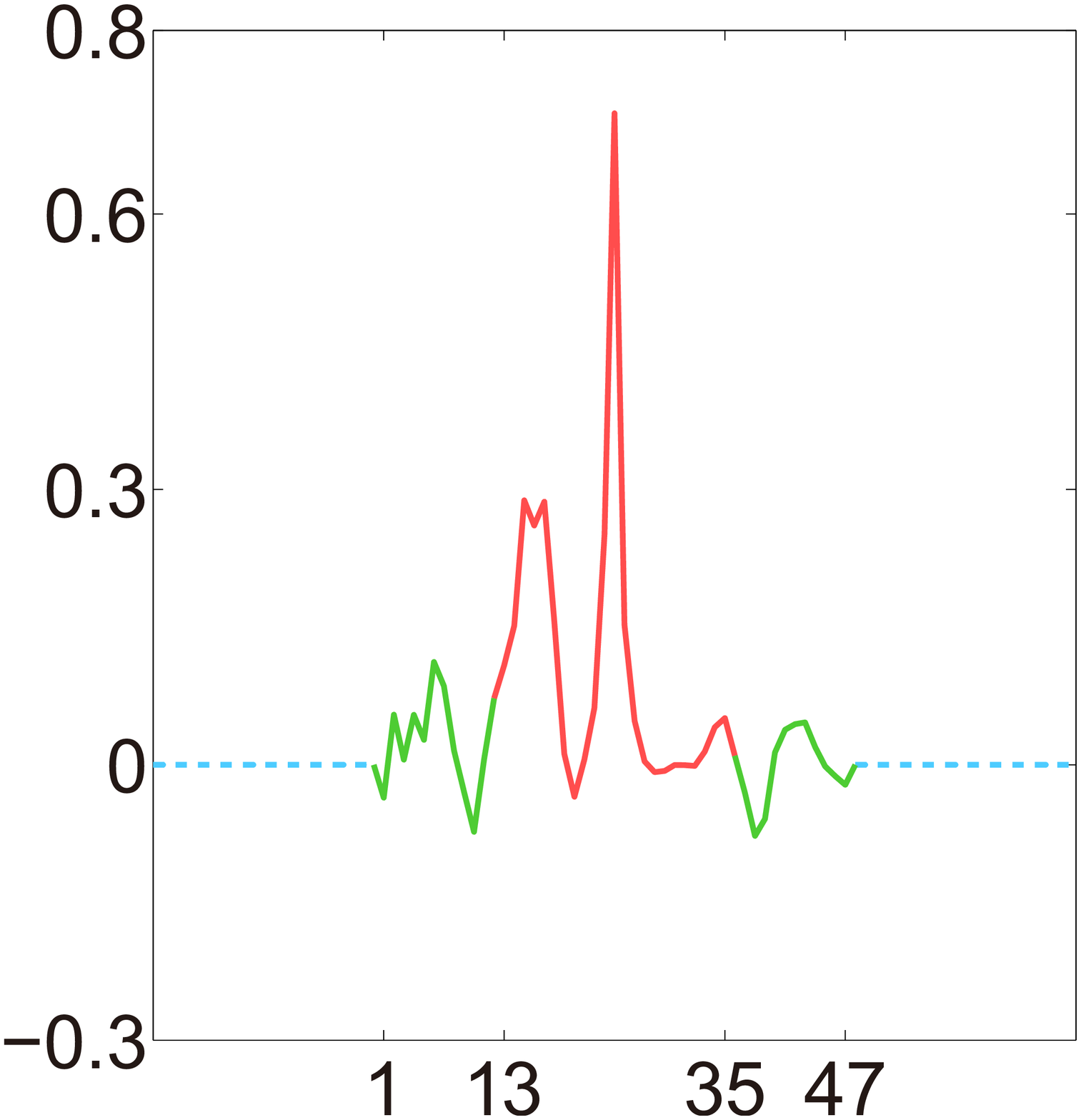} &

\includegraphics[width=.4\linewidth]{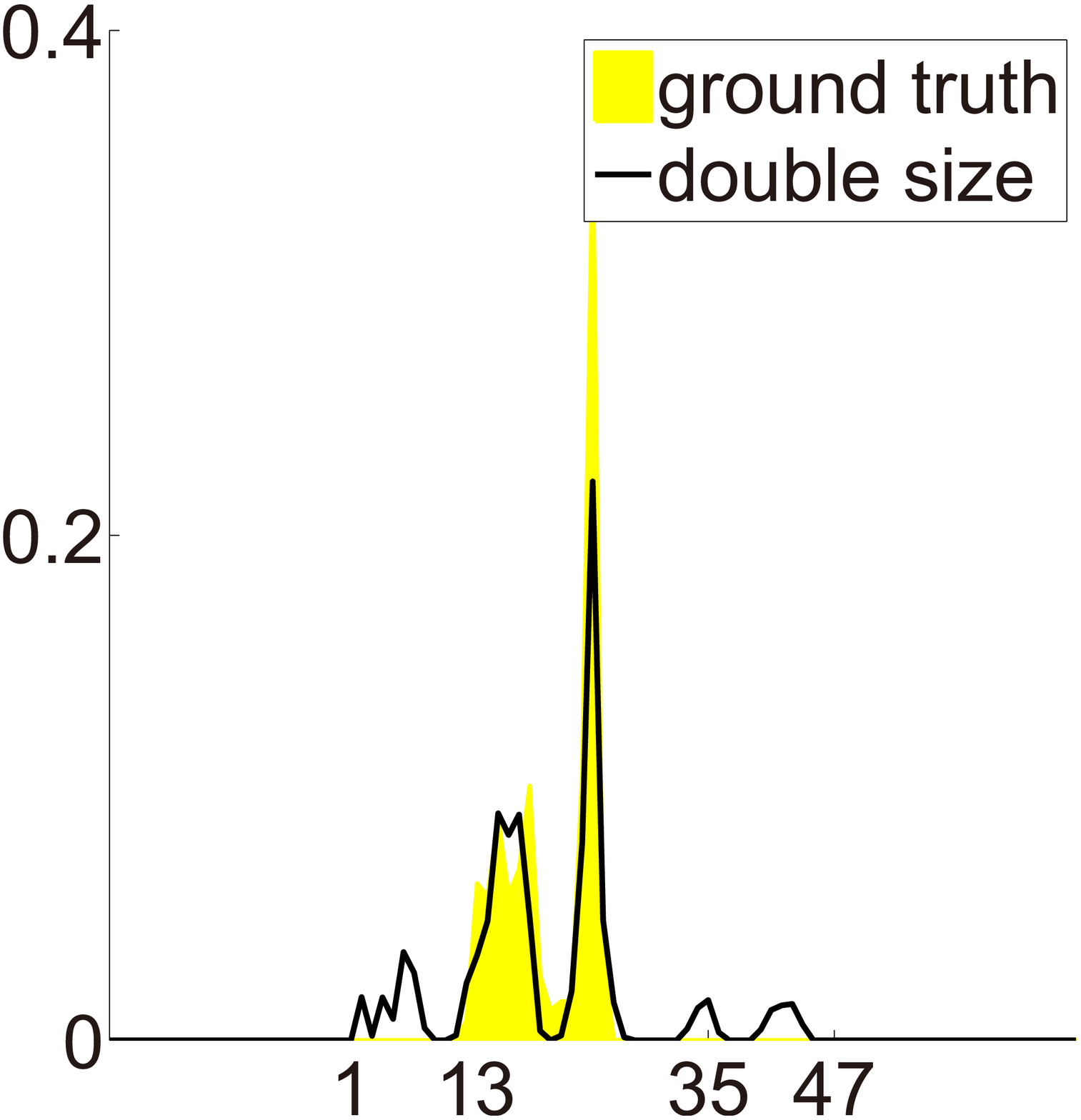} \vspace{0pt} \\
& (b) & (e) \\
&
\includegraphics[width=.4\linewidth]{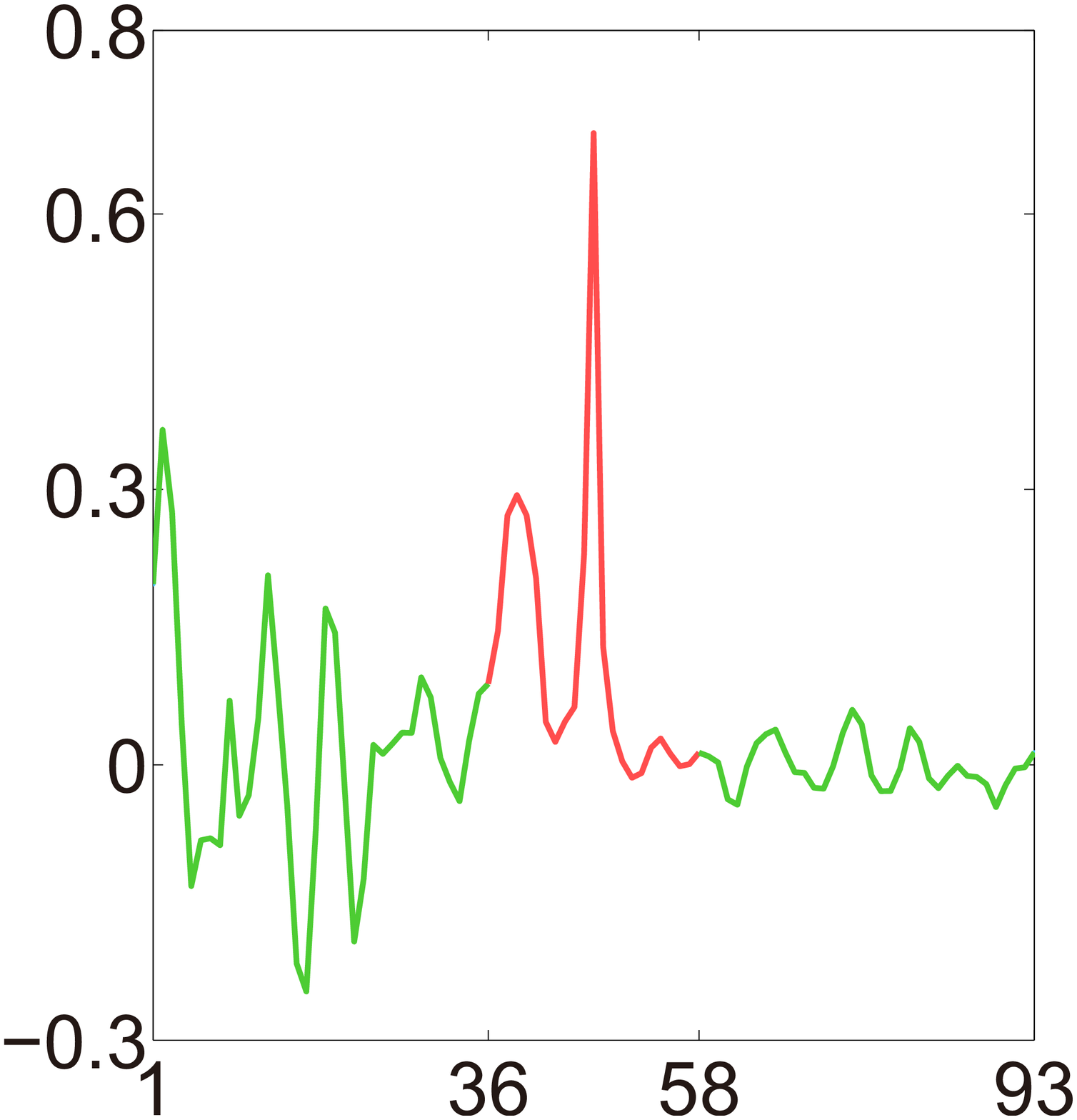} & 
\includegraphics[width=.4\linewidth]{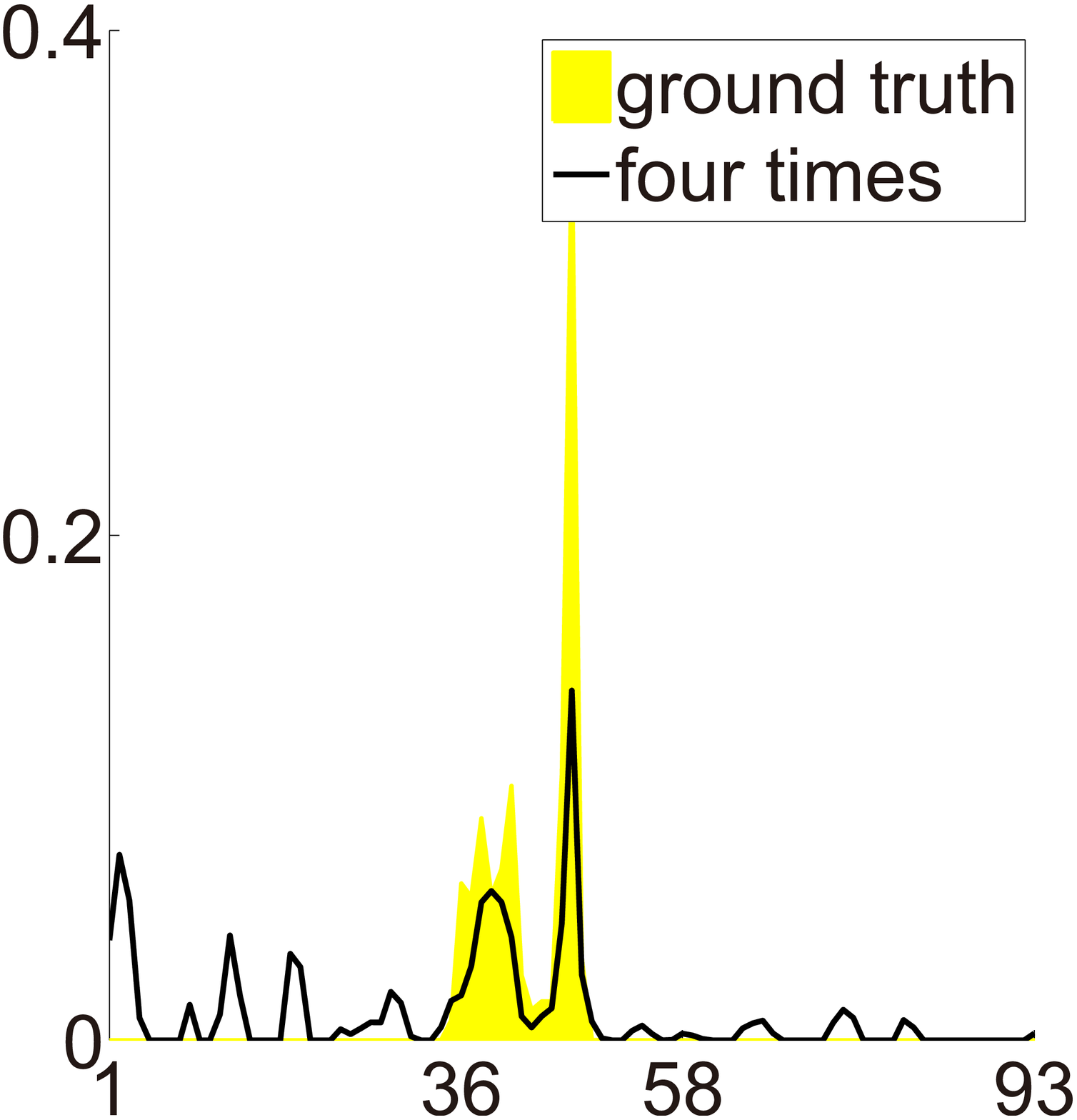}  \\
& (c) & (f) \\
& index &  index   \\
   
\end{tabular}\\
\end{center}
   \caption{Estimated kernels in 1D Blind deconvolution simulations. Left column are optimized kernels of different sizes after 50\small{th} iteration. Right column are corresponding normalized kernels of left after non-negativity and sum-to-one projections.  In this experiment, $\bm x$ is a 255$\times$1 vector extracted from a real image and the truth $\bm k$ is generated by marginalizing a $23\times23$ truth kernel from Levin's dataset~\cite{levin2009understanding}. The signal prior is $\frac{\ell_1}{\ell_2}$. This figure is recommended to view in color.}
\label{fig2}
\end{figure}

\subsection{Probability of larger-kernel effect}

Even if $\bm{\hat x}^{(i)}$ successfully iterates to truth $\bm{x}$, Claim 1 implicates
 the larger-kernel effect remains under the existence of random noise $\bm n$.
 We show
\begin{equation}\label{eq_pro_equals_1}
\Pr \left(\mathrm{rank}\left(\mathbf T_{\bm{x}}\left(L+2\right)\right) > \mathrm{rank} \left(\mathbf T_{\bm{x}} \left(L\right)\right)\right) = 1,
\end{equation}
under which, the \textit{inflating effect} holds for probability one in blind deconvolution.

\begin{figure*}[t]
\begin{center}
%\fbox{\rule{0pt}{2in} \rule{0.9\linewidth}{0pt}}
   \includegraphics[width=\linewidth]{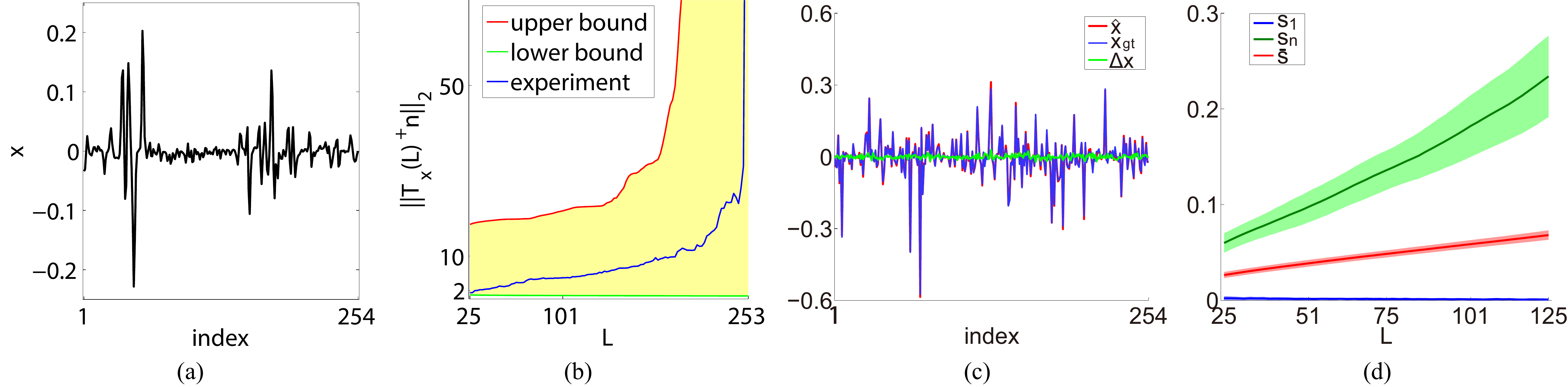}
\end{center}
   \caption{Quantitative simulations show that error increases with kernel size.
   (a) The extracted row from a clear image.
   (b) Singular boundaries of $\mathbf T_{\bm{x}}^\dagger$ and $\|\mathbf T_{\bm{x}}^\dagger \bm{n}\|$ with sampled $\bm n$. (c) Synthetic sparse signals. (d) The smallest, greatest and mean singular values of $(\mathbf T_{\bm x_{gt}}+\mathbf T_{\Delta \bm x})^\dagger \mathbf T_{\bm x_{gt}} - \mathbf I$.}
\label{fig3}
\end{figure*}

Above all, we have
\begin{equation}\label{eq_prob_trans}
\begin{aligned}
  &\Pr \left(\mathrm{rank}\left(\mathbf T_{\bm{x}}\left(L+2\right)\right) > \mathrm{rank} \left(\mathbf T_{\bm{x}} \left(L\right)\right)\right)
  \\ \ge &\Pr\left(\mathrm{rank}\left(\mathbf T_{\bm{x}}\left(M\right)\right)=M\right).
\end{aligned}
\end{equation}
Kaltofen and Lobo~\cite{kaltofen1996rank} proved
that for an M-by-M Toeplitz matrix composed of finite filed of $q$ elements,
\begin{equation}\label{eq_finite_prob}
  \Pr\big(\mathrm{rank}(\mathbf T_{M\times M}) = M\big) = 1-1/q.
\end{equation}
Herein, clear images are statistically sparse on derivative fields~\cite{olshausen1996emergence, weiss2007makes}, and elements of $\bm x$ are modeled to be continuous in hyper-Laplacian distributions~\cite{krishnan2009fast}:
\begin{equation}\label{eq_pdf}
p({x}) = \begin{cases}
\beta \exp\left(-\gamma |x|^{\alpha}\right) & ,x \in [-1, 1] \\
0 &, otherwise.
\end{cases}
\end{equation}
Then we get the following claim:
\begin{thm}
 $$ \Pr\big(\mathrm{rank}(\mathbf T_{\bm{x}}(M)) = M\big) = 1.$$
\end{thm}

\begin{proof}
See supplementary file.
\end{proof}
To now, we have shown that for $\mathbf{x}$ in sparse distribution, the \textit{inflating effect} happens almost surely.

\subsection{Quantification of error increment}
Assume $\mathbf{\hat x}$ iterates to ground truth $\mathbf{x}_{gt}$ during iterations. Then, for estimated kernel $\mathbf{\hat k}$, we have
\begin{equation}\label{eq_first_noise_k}
\begin{aligned}
  \hat{\bm{k}} =  \arg\min_{\bm k}{\|\mathbf T_{\bm x_{gt}}\bm k - \bm y\|^2}
  % = & \arg\min_k{\|\mathbf T_{x_{gt}}k - \mathbf T_{x_{gt}}\bm{k_{gt}} - \bm n\|^2}\\
  % = & \bm{k_{gt}} + \arg\min_k{\|\mathbf T_{x_{gt}}k-\bm n\|^2}\\
  =  \bm k_{gt} + \mathbf T_{\bm x_{gt}}^\dagger \bm n
\end{aligned}
\end{equation}
where $\dagger$ represents Moore-Penrose pseudo-inverse. Then,
\begin{equation}\label{eq_first_noise_ssd}
  SSD=\|\bm {\hat k} - \bm k_{gt} \|^2 = \|\mathbf T_{x_{gt}}^\dagger \bm n\|^2.
\end{equation}
Assume $\|\bm n\| = 1$, then
\begin{equation}\label{eq_first_bound}
  s_1(\mathbf T_{\bm x_{gt}}^\dagger) \le \|\mathbf T_{\bm x_{gt}}^\dagger \bm n\| \le s_n(\mathbf T_{\bm x_{gt}}^\dagger),
\end{equation}
where $s_1$ and $s_n$ represents the smallest and the greatest singular values, respectively.

The \textit{inflating effect} implicates that a larger kernel size amplifies the error in $\mathbf {\hat k}$ due to noise $\mathbf n$. To quantify this increment, we extracted a line $\bm x$ from a clear image in Levin's set~\cite{levin2009understanding} as shown in Figure ~\ref{fig3}(a), and plotted $s_1(\mathbf T_{\bm{x}}^\dagger)$ and $s_n(\mathbf T_{\bm{x}}^\dagger)$ with increasing kernel size $L$. We also generated normalized random Gaussian vectors $\bm n$ and compared $\|\mathbf T_{\bm x} \bm n\|$ to simulated boundaries of singular values (see Figure ~\ref{fig3}(b)). The error in $\mathbf{\hat k}$ increases hyper-linearly with kernel size.

In practice, nuances are expected between $\mathbf{\hat{x}}$ and $\mathbf{x}_{gt}$. Cho and Lee~\cite{cho2016convergence} indicated that $\mathbf{\hat x}$ should be regarded as a sparse approximation to $\mathbf{x}_{gt}$, not the ground truth. Hence,
\begin{equation}\label{eq_not_same}
\mathbf{\hat x} = \mathbf{x}_{gt} + \Delta \mathbf x,
\end{equation}
which yields implicit noise~\cite{shan2008high}.
Assume $\mathbf n = \mathbf 0$, then,
\begin{equation}\label{eq_second_noise_conv}
  \begin{aligned}
    \mathbf{\hat x} \otimes \mathbf{k}_{gt}  = \mathbf x_{gt} \otimes \mathbf k_{gt} + \Delta \mathbf x \otimes \mathbf k_{gt}
    % \\ & = &&x_{gt} * k_{gt} + n_{\Delta x},
  \end{aligned}
\end{equation}
and
\begin{equation}\label{eq_second_k}
  \begin{aligned}
    \bm{\hat k} = & \arg\min_{\bm k}\|\mathbf T_{\bm {\hat x}} \bm k - \bm y\|^2 \\
    % = & \arg\min_k\|(\mathbf T_{x_{gt}} +\mathbf T_{\Delta x})k - \mathbf T_{x_{gt}}\bm{k_{gt}}\|^2 \\
    = & (\mathbf T_{\bm x_{gt}} + \mathbf T_{\Delta \bm x})^\dagger \mathbf T_{\bm x_{gt}}\bm{k}_{gt}.
  \end{aligned}
\end{equation}
Then,
\begin{equation}\label{eq_second_ssd}
  SSD=\|\left((\mathbf T_{\bm x_{gt}}+\mathbf T_{\Delta \bm x})^\dagger \mathbf T_{\bm x_{gt}}-\mathbf I\right)\bm{k}_{gt}\|^2.
\end{equation}
To quantify how singular values of $(\mathbf T_{\bm x_{gt}}+\mathbf T_{\Delta \bm x})^\dagger \mathbf T_{\bm x_{gt}} - \mathbf I$ changes with kernel size, we simulated 100 times, in each of which we generated a stochastic sparse signal $\bm x_{gt}$ with length 254 under PDF in (\ref{eq_pdf}) with $M = 254$, $\gamma = 10$ and $\alpha = 0.5$, and generated random Gaussian vector $\Delta \bm{x}$ where $ \|\Delta \bm x\| = \|\bm x_{gt}\| / 100$. Figure ~\ref{fig3}(c) shows one example of generated $\bm x_{gt}$ and $\Delta \bm{x}$. Figure ~\ref{fig3}(d) shows means and standard deviations of $s_1$, $s_n$ and $\bar s$, which is the average of singular values, of simulated $(\mathbf T_{\bm x}+\mathbf T_{\Delta \bm x})^\dagger \mathbf T_{\bm x} - \mathbf I$ on $L$. The error of $\mathbf{\hat k}$ is expected to grow with kernel size even $\mathbf n =0$.
% \begin{figure}[h]
% \begin{center}
% %\fbox{\rule{0pt}{2in} \rule{0.9\linewidth}{0pt}}
%    \includegraphics[width=0.95\linewidth]{siyao4.pdf}
% \end{center}
%    \caption{
%    (a) Simulated sparse signal examples. (b) Simulated singular values of $(\mathbf T_{\bm x}+\mathbf T_{\Delta \bm x})^\dagger \mathbf \mathbf \mathbf \mathbf \mathbf T_{\bm x} - I$.
%    }
% \label{fig4}
% \end{figure}

\begin{figure*}[t]
\begin{center}
   \includegraphics[width=0.9\linewidth]{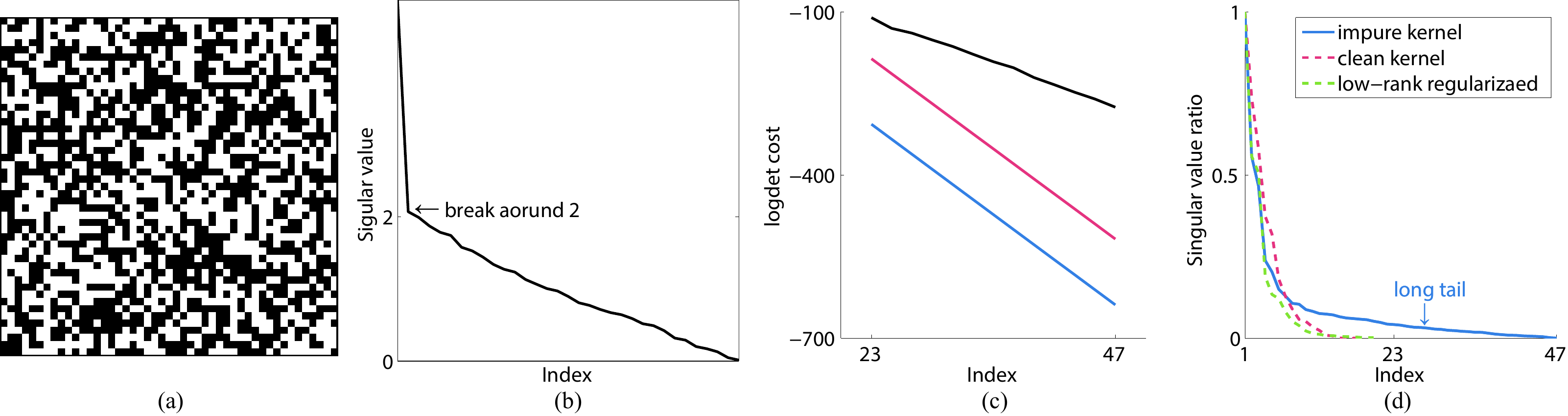}
\end{center}
   \caption{Singular values of clean kernels and noisy matrices. (a) The support domain (black) of a $47\times47$ random positive half Gaussian noise matrix. (b) The distribution of singular values of (a). (c) $\log\det$ costs of random Gaussian noise matrix (black), the truth kernel from~\cite{levin2009understanding} after zero-padding (red), and Gaussian PSF with $\text{size}/6$ standard deviation (blue) on kernel size. (d) Scaled (maximum to 1) singular value distributions of clean, impure and regularized kernels. }
\label{fig6}
\end{figure*}

\section{Low-rank regularization}
Blind deconvolution is an ill-posed problem for lacking sufficient information. Without regularization, MAP degrades to Maximum Likelihood (ML), which yields infinite solutions~\cite{levin2009understanding}. As prior information, kernel regularization should be designed to compensate the shortage of ML and to guide the optimization to expected results. 
Great amount of studies focus on image regularization to describe natural images, \eg,  Total Variation (TV-$\ell_1$)~\cite{levin2007image,wang2008new,ren2015fast}, hyper-Laplacian~\cite{krishnan2009fast}, dictionary sparsity~\cite{zhang2010bregmanized,hu2010single}, patch-based low rank prior~\cite{ren2016image}, non-local similarity~\cite{dong2013nonlocal} and deep discriminative prior~\cite{li2018learning}.

Unfortunately, kernel optimization doesn't attract much attention of the literature. Previous works adopted various kernel regularizations, e.g., $\ell_2$-norm~\cite{xu2010two, gong2016blind, cho2009fast, xu2013unnatural,pan2016blind}, $\ell_1$-norm~\cite{krishnan2011blind, shan2008high, pan2016robust} and $\ell_\alpha$-norm $(0<\alpha<1)$~\cite{zuo2015discriminative}, which, however, generally treated kernel regularization as an accessory and lacked a detailed discussion.

The larger-kernel effect is yielded by noise in ultra-sized kernels. Figure ~\ref{fig1} and Figure ~\ref{fig2} show that without kernel regularization, the main bodies of estimated kernels can emerge clearly, but increasing noises take greater amounts when $k$ is larger. To constrain $\hat k$ to be clean, regularization $h$ is expected to distinguish noise from ideal kernels efficiently.

To suppress the noise in estimated kernels, we take low-rank regularization on $k$ such that k-step \eqref{eq4} becomes
\begin{equation}\label{eq_low_rank_opt}
\mathbf{\hat{k}}^{(i+1)} = \arg\min_{\mathbf k}{\left(\|\mathbf{\hat{x}}^{(i+1)}\otimes \mathbf k- \mathbf y\|^2+\sigma \mathrm{rank}\left(\mathbf k\right)\right)}.
\end{equation}

Because the direct rank optimization is an NP-hard problem, continuous proximal functions are required. Fazel~\etal~\cite{fazel2003log} proposed
\begin{equation}\label{heuristic_proxy}
  \log\det\left(\mathbf X + \delta \mathbf I\right)
\end{equation}
as a heuristic proxy for
 $\mathbf X\in \mathbb{S}^N_+ $ where $\mathbf I$ is the N-by-N identity matrix and $\delta$
 is a small positive number. 
 % A 1D case is shown in Figure ~\ref{fig5}.
% \begin{figure}[h]
% \begin{center}
%    \includegraphics[width=0.4\linewidth]{siyao5.pdf}
% \end{center}
%    \caption{1D demonstration of the rank function, the nuclear norm and $\log\det$.}
% \label{fig5}
% \end{figure}

To allow this approximation to play a role in general matrices, the low-rank object is substituted to $(\mathbf X \mathbf X^T)^{1/2}$~\cite{dong2014compressive}. The regularization function then becomes
\begin{equation}\label{heuristic_proxy2}
  h(\mathbf X)=\log{\det ((\mathbf X \mathbf X^T)^{\frac{1}{2}}+\delta \mathbf I)} = \sum_j \log(s_i + \delta),
\end{equation}
where $s_i$ is the $i$-th singular value of $X$.

Taking low-rank regularization on kernels is motivated by a generic phenomenon of noise matrices~\cite{andrews1977digital}.
Figure ~\ref{fig6}(a-b) shows a non-negative Gaussian noise matrix and its singular values in decreasing order.  For a noise matrix, where light and darkness alternate irregularly, the distribution of singular values decays sharply at lower indices; then, it breaks and drag a relatively long and flat tail to the last.
In contrast, ideal kernels respond much lower to $\log\det$ regularization (see Figure ~\ref{fig6}(c)).
Based on this fact, noise matrices are distinguished by high $\log\det$ cost from real kernels. Figure ~\ref{fig6}(d) shows that singular values of a low-rank regularized kernel are distributed similarly as the ground truth, compared with the impure one.

One intelligible explanation on the low-rank property of ideal kernels is the continuity of blur motions. Rank of a matrix equals the number of independent rows or columns; it reversely reflects how similar these rows or columns are. Speed of a camera motion is deemed to be continuous~\cite{fang2014separable}. Hence, the local trajectory of a blur kernel emerges similar to neighbor pixels, which is measured in a low value by the continuous proxy of rank.

\begin{figure}[t]
\begin{center}
   \includegraphics[width=0.8\linewidth]{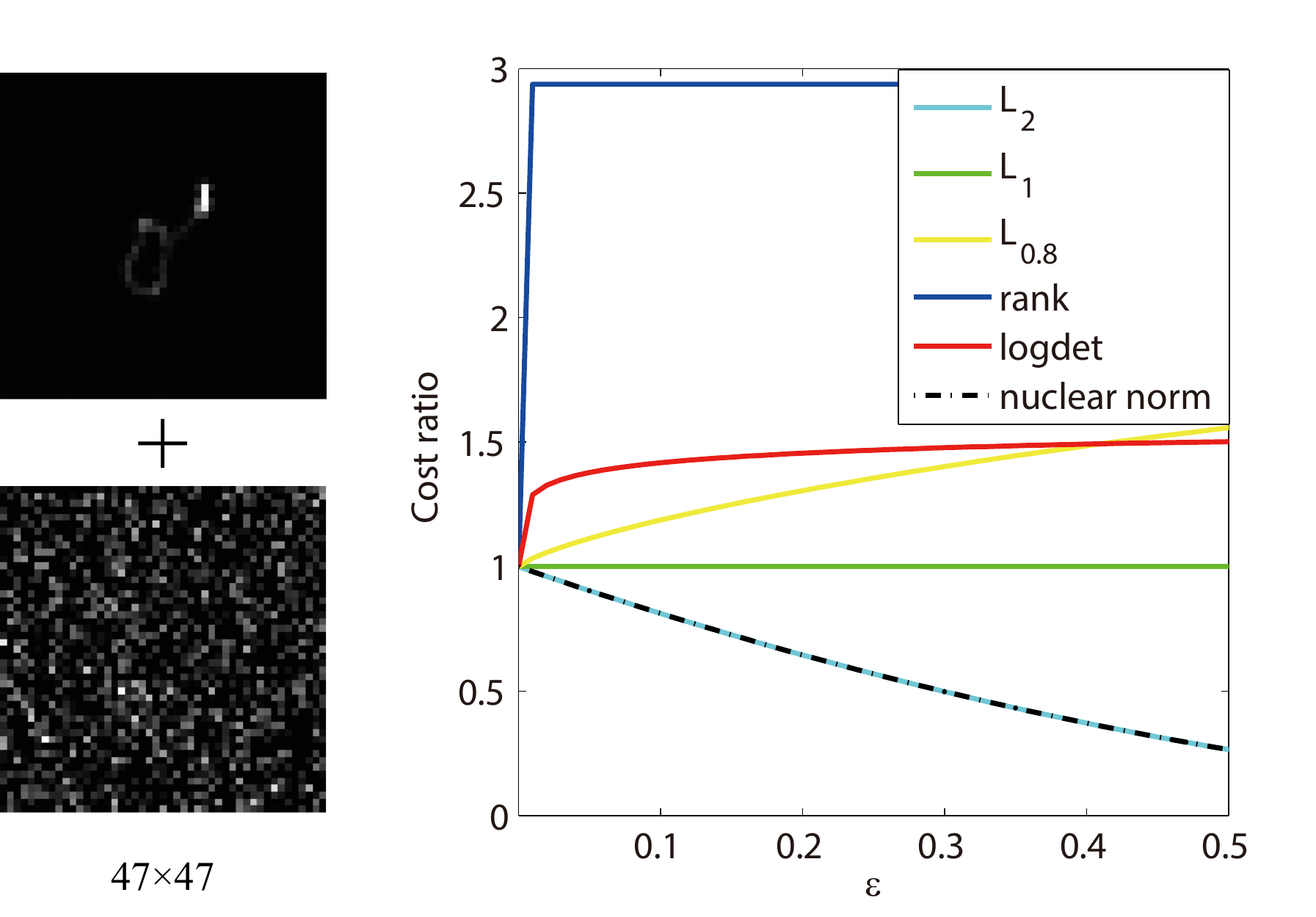}
\end{center}
   \caption{Comparison on respond to noise. The cost ratio is calculated as $1+\frac{\text{cost}(\epsilon)-\text{cost}(0)}{|\text{cost}(0)|}$. This figure is recommended to view in color.}
\label{fig7}
\end{figure}
Compared to previous $\ell_{\alpha}$ norms, low-rank regularization responds more efficiently to noise. To illustrate this point,
we generated a noisy kernel by adding a small percentage ($\epsilon$) of non-negative Gaussian
noise and $1-\epsilon$ of the real kernel. Figure ~\ref{fig7} shows that the low-rank cost rapidly adjust favorably to the noise but $\ell_{\alpha}$ norms fail. That is because $\ell_\alpha$ only takes statistical information. An extreme example consists of disrupting a truth kernel and randomly reorganizing its elements, with $\ell_\alpha$ cost unchanged. In contrast, rank (singular values)  corresponds to structural information.

\section{Optimization}

Function $\log\det$ is non-convex (and it is actually concave on $\mathbb{S}_+$). To solve the low-rank regularized least
squares ~\eqref{eq4}, % with $h$ defined in (\ref{heuristic_proxy2}),
we introduce an auxiliary variable $\Psi = k$ and reformulate the optimization into
\begin{equation}\label{2_optimiz}
\begin{aligned}
  \min_{\mathbf k, \mathbf \Psi} & \left(\mathbf{\|\hat x}^{(i+1)}\otimes \mathbf \Psi-\mathbf y\|^2  +
\sigma\log\det\left(\left(\mathbf k \mathbf k\right)^{\frac{1}{2}}+\delta\mathbf{I}\right)\right) \\
\text{s.t. } & \mathbf \Psi  = \mathbf k
\end{aligned}
\end{equation}

Using the Lagrange method, (\ref{2_optimiz}) is solved by two alternate sub-optimizations
\begin{equation}\label{eq:two_step}
\left\{
             \begin{aligned}
             \mathbf{\hat \Psi}^{(j+1)} &= \arg\min_{\mathbf \Psi} {\|\mathbf{\hat x}^{(i+1)}\otimes \mathbf \Psi-\mathbf y\|^2 + \mu \|\mathbf \Psi -
             \mathbf {\hat k}^{(j)}\|^2  }\\
             \mathbf{\hat k}^{(j+1)} &= \arg\min_{\mathbf k}{
             \frac{1}{2\tau}\|\mathbf k-\mathbf{\hat\Psi}^{(j+1)}\|^2+\sigma h(\mathbf k)}
             \end{aligned}
\right.
\end{equation}
where $j$ is the iteration number while $\mu$ and $\tau$ are trade-off parameters.

The $\Psi$-substep is convex and accomplished using the Conjugate Gradient (CG) method. For $k$-substep,
low rank is adopted with limit; otherwise, the regularization may change the
main body of kernel---an extreme result is $\mathbf{\hat k} = 0$. Thus, our strategy is to lower
the rank  at $\mathbf{\hat \Psi}$ locally. Using the first-order Taylor expansion of $h$ at fixed matrix $\mathbf Z$:
\begin{equation}\label{taylor}
  \begin{aligned}
  h_{\mathbf Z}(\mathbf X)=h\left(\mathbf Z\right) + \sum_i\frac{s_i-\hat{s}_i}{\hat{s}_i+\delta},
  % &h\left(Z\right)+
  % \tr\left(\left(ZZ^T\right)^\frac{1}{2}+\delta\mathrm{I}\right)^{-1}\left(\left(XX^T\right)^\frac{1}{2}-\left(ZZ^T\right)^\frac{1}{2}\right)\\
  \end{aligned}
\end{equation}
where $\hat{s}_i$ is the $i$-th eigenvalue of $\mathbf Z$, the k-substep in \eqref{eq:two_step} is transformed into an iterative optimization
\begin{equation}\label{final_eq}
\mathbf k^{(t+1)} = \arg\min_{\mathbf k}
             \left(\frac{1}{2\tau}\|\mathbf k-\mathbf{\hat\Psi}^{(j+1)}\|^2+\sigma h_{\mathbf k^{(t)}}\left(\mathbf k\right) \right)
\end{equation}
where $t$ is the inner iteration number. For convenience, we set $\sigma$ as a flag (if $\sigma=0$, the k-substep will be skipped) and only tuned $\tau$ as the trade-off parameter.

Define the proximal mapping of function $\phi$ as follows:
\begin{equation}\label{prox1}
  \prox_\phi(v) = \arg\min_u\left(\frac{1}{2}\|u-v\|^2 + \phi(u)\right).
\end{equation}
Dong~\etal~\cite{dong2014compressive} proved that one solution to the proximal mapping of $\tau h_{\mathbf Z}$ is
\begin{equation}\label{prox2}
  \prox_{\tau
h_{\mathbf Z}}\left(\mathbf X\right)=\mathbf U\left(\mathbf \Sigma-\tau\mathrm{diag}\left(\bm{w}\right)\right)_+\mathbf V^T
\end{equation}
where
$\mathbf U\mathbf \Sigma \mathbf V^T$ is SVD of $\mathbf X$, $w_i = 1/\left(\hat{s}_i+\delta \right) $
and
${(\cdot)}_+=\max{\{ \cdot,0 \}}$.
Local low-rank optimization is implemented as iterations via the given parameter $\tau$ (see Algorithm~\ref{algk}).
In our implementation, $\mu$ is designed to exponentially grow with $j$ to allow more freedom of $\mathbf{\hat \Psi}$ for early iterations.

\emph{Overall Implementation.} We took deconvolution sechme in~\cite{krishnan2011blind} where $g = \ell_1/\ell_2$ (but with small modification) and applied non-blind deconvolution method proposed in~\cite{krishnan2009fast}.

\begin{algorithm}[t]
\caption{Updating k with low-rank regularization}\label{algk}
\begin{algorithmic}[1]
\Require{$\mathbf x$, $\mathbf y$, $\mu$, $\tau$, $OuterIterMax$, $CGIterMax$, $innerIterMax$}
\Ensure{$\mathbf{\hat k}$}
\For{$j\gets0$ to $OuterIterMax - 1$}
\If{$j=0$}
\State{$\mathbf{\hat \Psi}^{(j+1)}\gets \min_{\mathbf\Psi} \|\mathbf x\otimes\mathbf \Psi-\mathbf y\|^2 $ using CG
with maximum $CGIterMax$ iterations}
\Else
\State{$\mu^{(j)} \gets \mu \left(e^j/e^{OuterIterMax}\right)$}
\State{\small{$\mathbf{\hat \Psi}^{(j+1)}\gets \min_{\mathbf \Psi}{ \|\mathbf x \otimes \mathbf \Psi - \mathbf y\|^2 + \mu^{(j)} \|\mathbf \Psi - \mathbf k^{(j)}\|^2 }$} using CG with $CGIterMax$ iterations}
\EndIf
\State{Initializing $\mathbf k^{(0)}$ with all singular values equal to 1}
\For{$t\gets 0$ to $innerIterMax - 1$}
\State{$\mathbf k^{(t+1)}\gets \prox_{\tau h_{\mathbf k^{(t)}}}\left(\mathbf{\hat \Psi}^{(j)}\right)$}
\EndFor{}
\State{$\mathbf{\hat k}^{(j+1)}\gets \max\{\mathbf k^{(innerIterMax)}, 0\}$}
\State{$\mathbf{\hat k}^{(j+1)}\gets  \mathbf{\hat k}^{(j+1)}/\sum \mathbf{\hat k}^{(j+1)}$}
\EndFor
\State{$\mathbf{\hat k}\gets \mathbf{\hat k}^{(outerIterMax)}$}
\end{algorithmic}
\end{algorithm}

\begin{algorithm}[t]
\caption{Blind Deconvolution (single-scaling version)}\label{algall}
\begin{algorithmic}[1]
\Require{blurry image $\mathbf y$, kernel size $L$, $\lambda$, $\eta$, $\tau$, $IterMax$}
\Ensure{clear image $x$ and degradation kernel $k$}
\State{$\mathbf y\gets\left[\mathbf \nabla_h{\mathbf y}, \mathbf \nabla_v{\mathbf y}\right]$}
\State{Initialize $\mathbf x \gets \mathbf y$}
\State{Initialize $\mathbf k$ with an $L\times L$ zero matrix adding [0.5 0.5] in the center}
\For{$t\gets 1$ to $IterMax$}
\State{Update $\mathbf x$ using Algorithm 3 in~\cite{krishnan2011blind}}
\State{Update $\mathbf k$ using Algorithm~\ref{algk}}
\EndFor{}
\State{$x\gets$ Non-blind deconvolution $(\mathbf k, \mathbf y)$}
\end{algorithmic}
\end{algorithm}

\section{Experimental Results}
In this section, we first discuss the effects of low rank-based regularization, then evaluate the proposed method on benchmark datasets, and finally demonstrate its effectiveness on real-world blurry images. 
The source code is available at \url{https://github.com/lisiyaoATbnu/low_rank_kernel}.

\begin{figure}[b]
\centering
  \setlength{\tabcolsep}{3pt}
  % Requires \usepackage{graphicx}
  \small
  \begin{tabular}{ccc}

  \includegraphics[width=.29\linewidth]{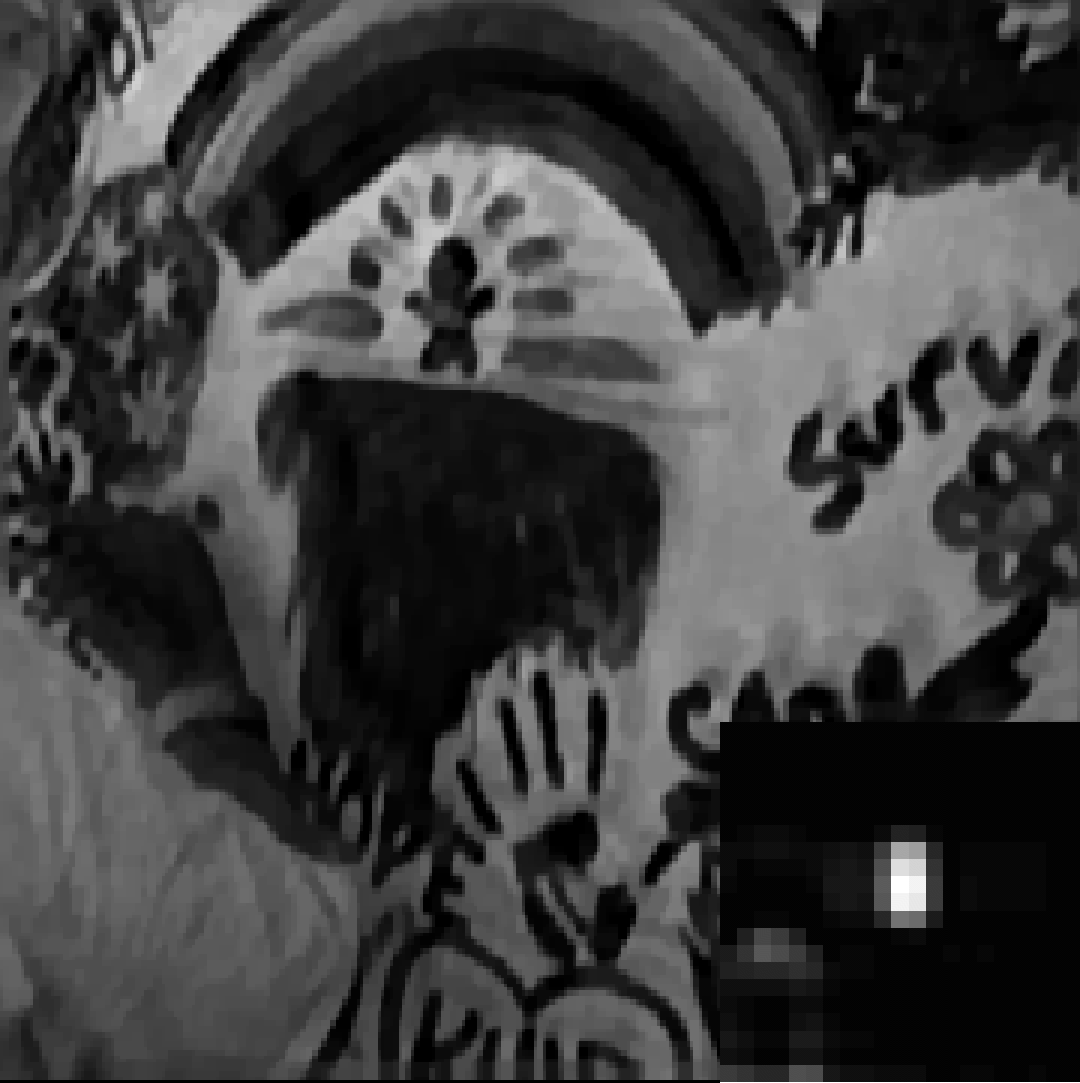} &
\includegraphics[width=.29\linewidth]{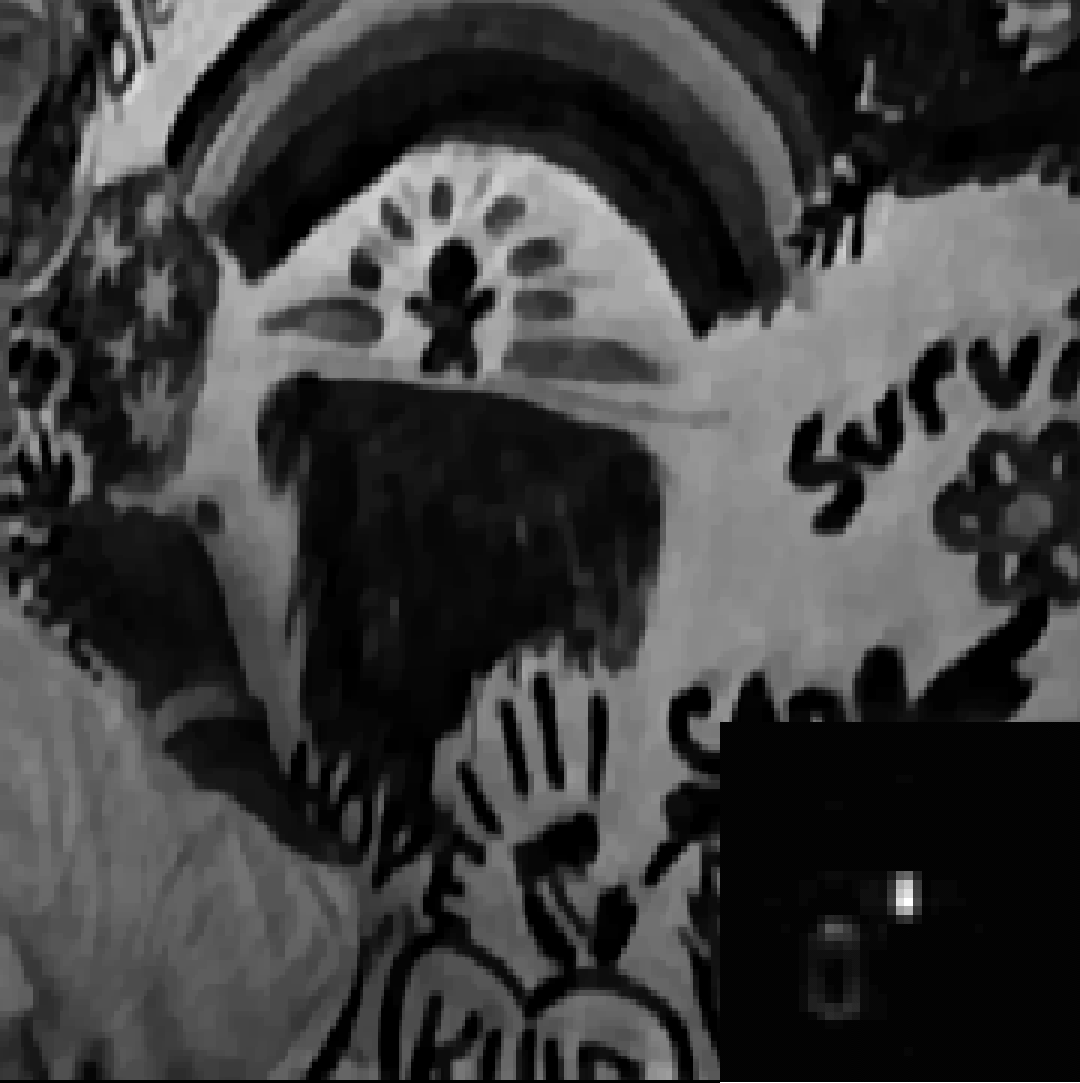} &
\includegraphics[width=.29\linewidth]{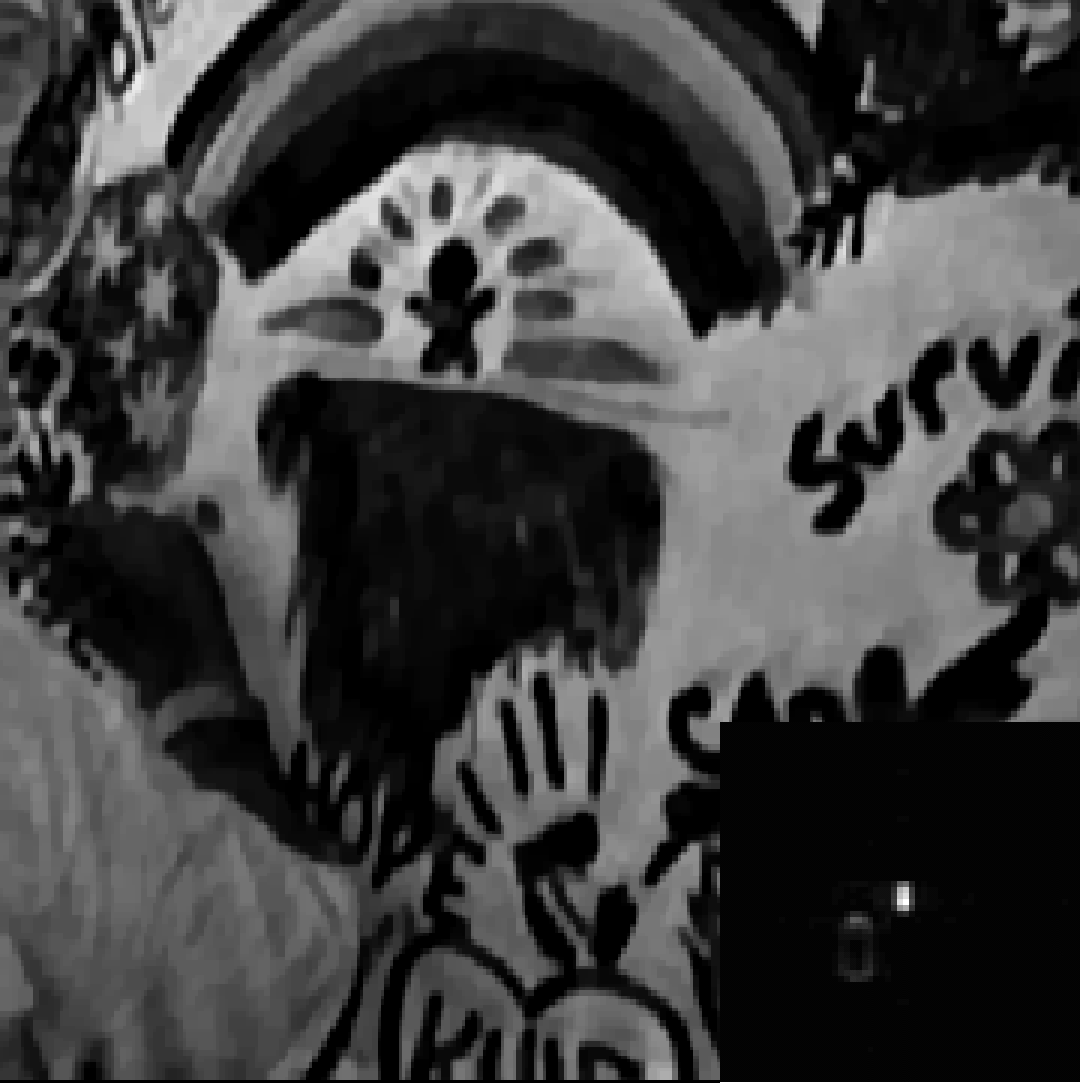} \vspace{-4pt} \\

 size=23, err=1.55 & size=47, err=1.56 & size=69, err=2.14 \vspace{1pt}\\

\end{tabular}\\
  \caption{Deblurring results using low-rank regularization.}\label{fig8}

\end{figure}

\begin{figure*}[t]
\centering
  \setlength{\tabcolsep}{1pt}
  % Requires \usepackage{graphicx}
  \small
  \begin{tabular}{ccccccccccccccccccccc}

   \multicolumn{3}{c}{\includegraphics[width=0.13\linewidth]{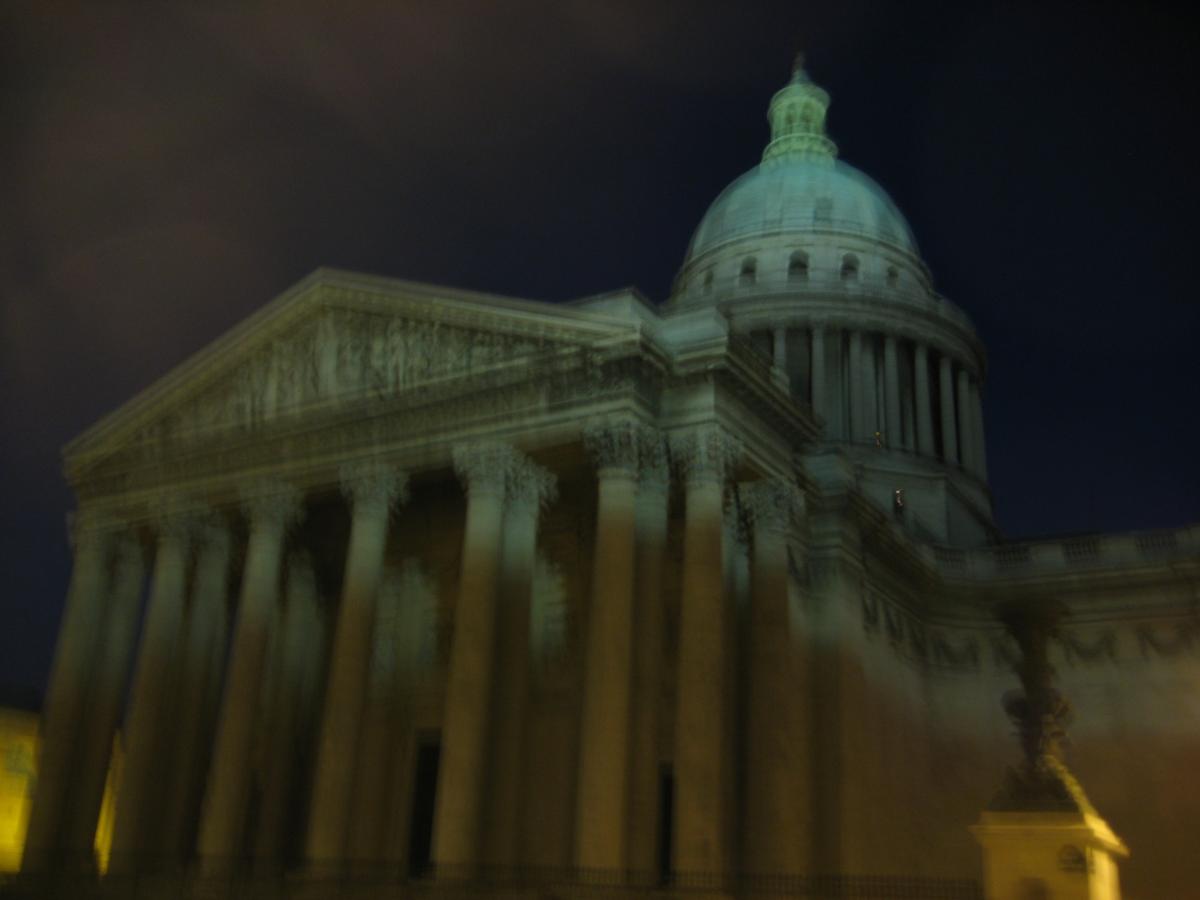}}
  &\multicolumn{3}{c}{\includegraphics[width=0.13\linewidth]{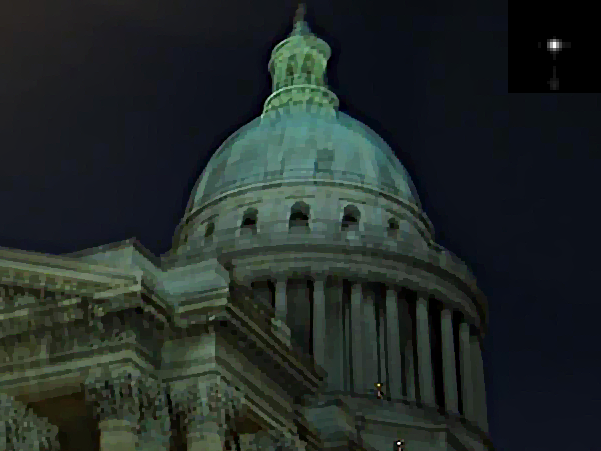}}
  &\multicolumn{3}{c}{\includegraphics[width=0.13\linewidth]{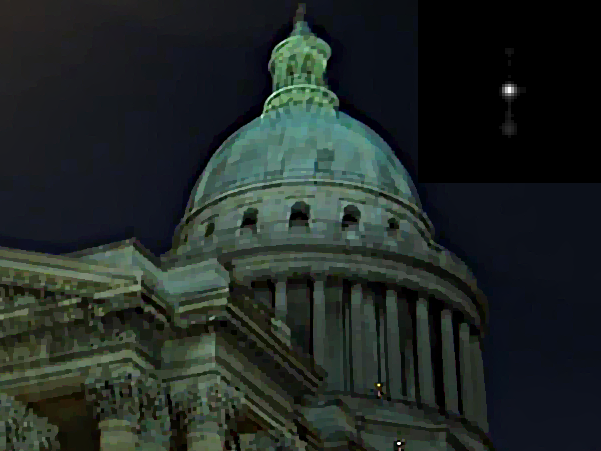}}
  &\multicolumn{3}{c}{\includegraphics[width=0.13\linewidth]{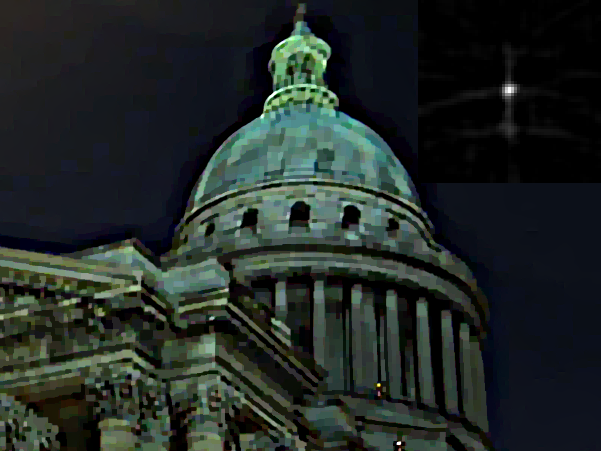}}
  &\multicolumn{3}{c}{\includegraphics[width=0.13\linewidth]{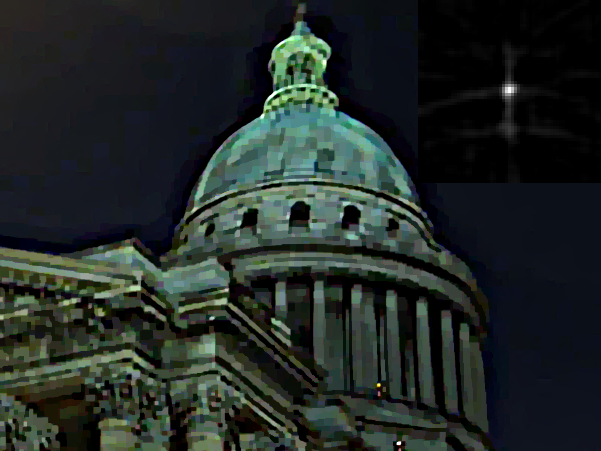}}
  &\multicolumn{3}{c}{\includegraphics[width=0.13\linewidth]{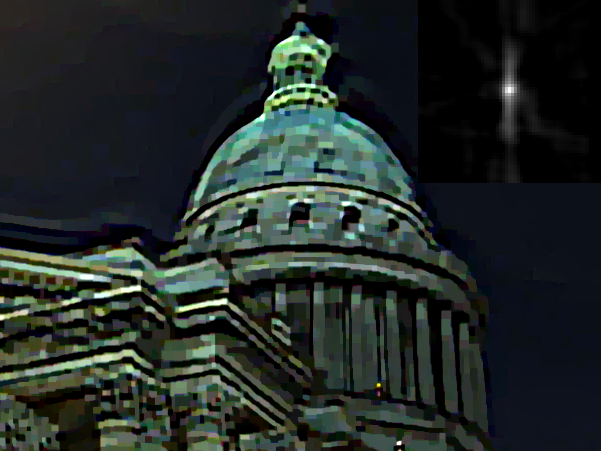}}
  &\multicolumn{3}{c}{\includegraphics[width=0.13\linewidth]{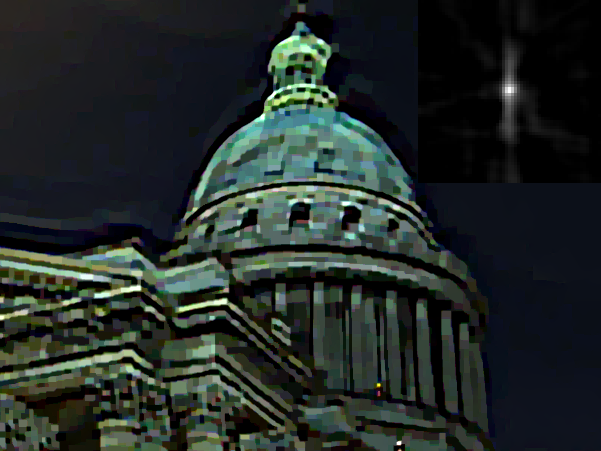}} \vspace{-1pt}\\

   \multicolumn{3}{c}{\includegraphics[width=0.13\linewidth]{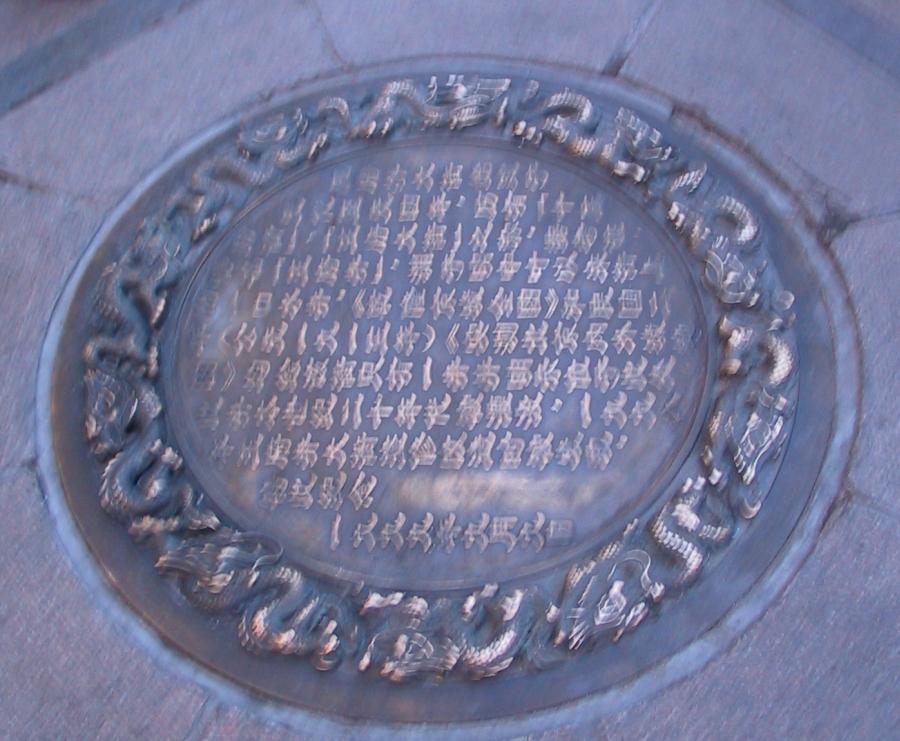}}
  &\multicolumn{3}{c}{\includegraphics[width=0.13\linewidth]{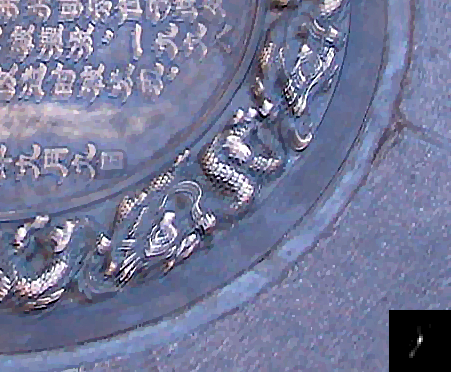}}
  &\multicolumn{3}{c}{\includegraphics[width=0.13\linewidth]{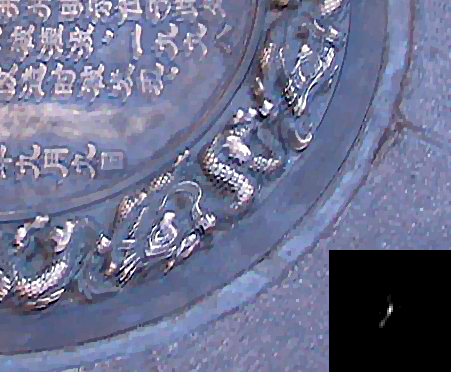}}
  &\multicolumn{3}{c}{\includegraphics[width=0.13\linewidth]{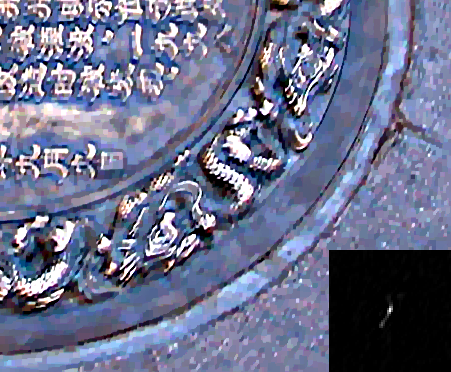}}
  &\multicolumn{3}{c}{\includegraphics[width=0.13\linewidth]{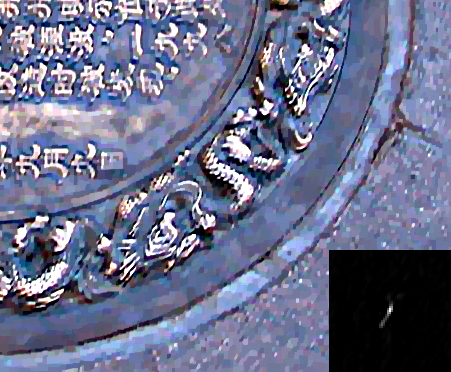}}
  &\multicolumn{3}{c}{\includegraphics[width=0.13\linewidth]{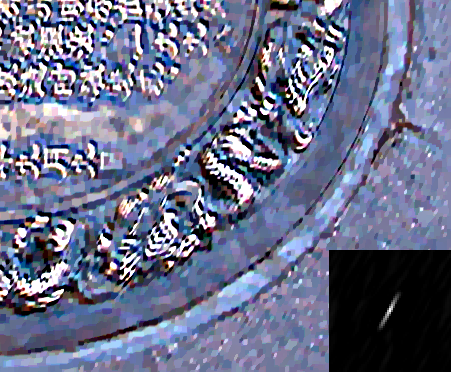}}
  &\multicolumn{3}{c}{\includegraphics[width=0.13\linewidth]{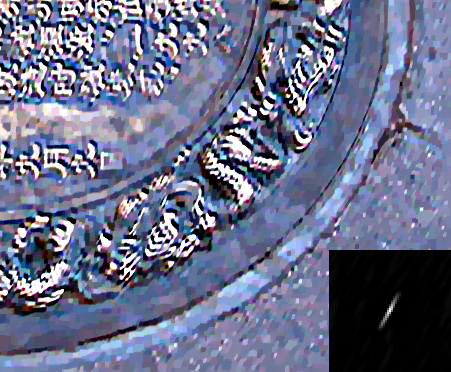}} \vspace{-1pt}\\
    
   \multicolumn{3}{c}{\includegraphics[width=0.13\linewidth]{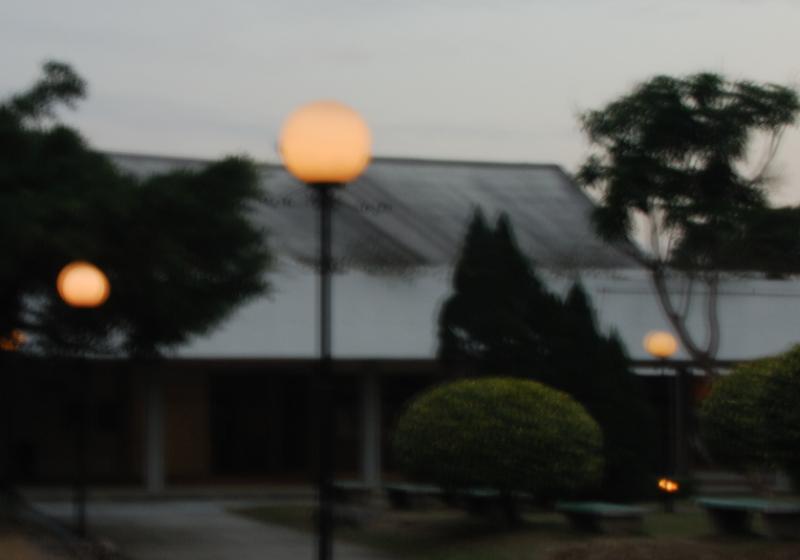}}
  &\multicolumn{3}{c}{\includegraphics[width=0.13\linewidth]{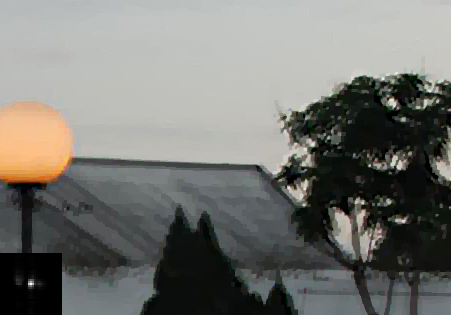}}
  &\multicolumn{3}{c}{\includegraphics[width=0.13\linewidth]{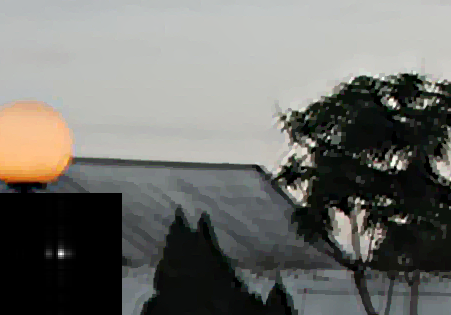}}
  &\multicolumn{3}{c}{\includegraphics[width=0.13\linewidth]{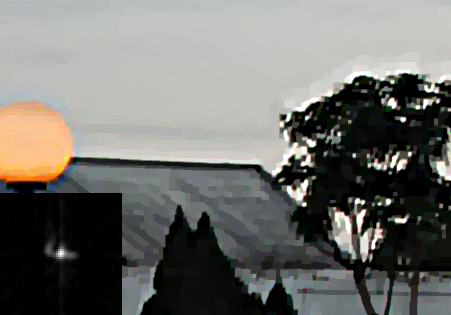}}
  &\multicolumn{3}{c}{\includegraphics[width=0.13\linewidth]{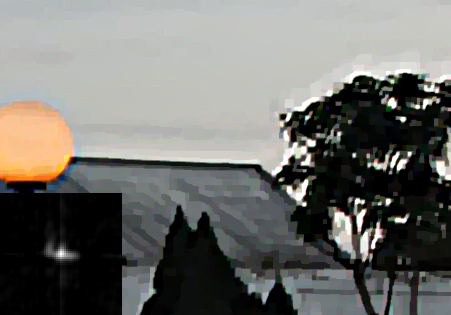}}
  &\multicolumn{3}{c}{\includegraphics[width=0.13\linewidth]{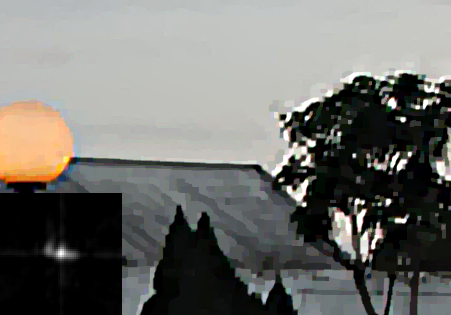}}
  &\multicolumn{3}{c}{\includegraphics[width=0.13\linewidth]{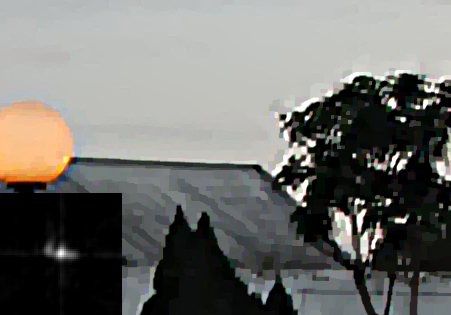}} \vspace{-1pt}\\

  \multicolumn{3}{c}{Blurry} 
  & \multicolumn{3}{c}{$31\times31$} 
  & \multicolumn{3}{c}{Low rank} 
  & \multicolumn{3}{c}{None} 
  & \multicolumn{3}{c}{$\ell_2$}
  & \multicolumn{3}{c}{$\ell_1$} 
  & \multicolumn{3}{c}{$\ell_\alpha$} \vspace{10pt}\\

\end{tabular}\\
   \caption{Comparison of different kernel priors on real-world images. Large kernel size is $61\times61$. It's recommended to zoom in.}
\label{fig9}
\end{figure*}

\begin{figure}[t]
\centering
   \setlength{\tabcolsep}{0.2pt}
  % Requires \usepackage{graphicx}
  \small
  \begin{tabular}{cc}
  \includegraphics[width=0.49\linewidth]{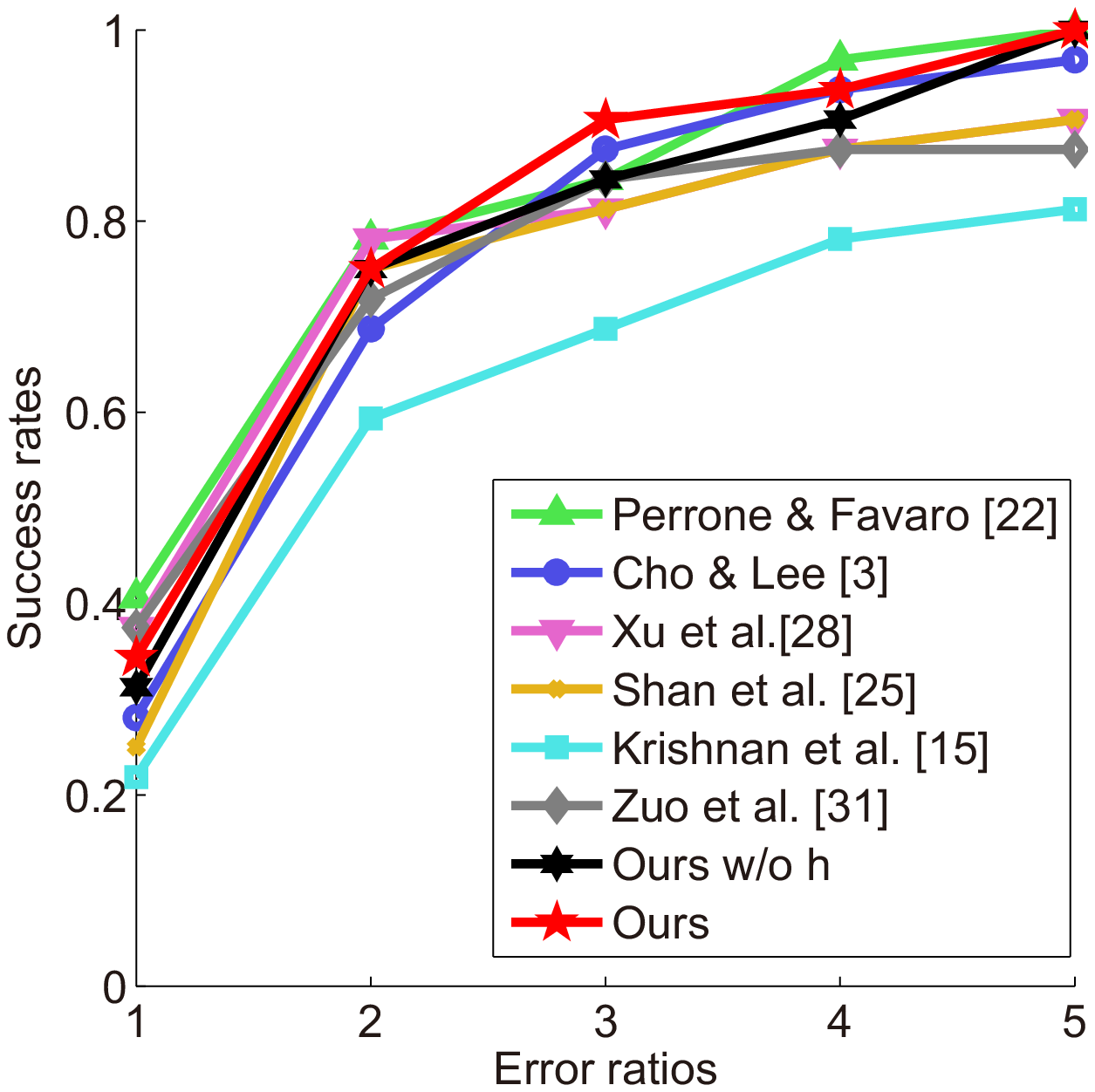}
  & \includegraphics[width=0.49\linewidth]{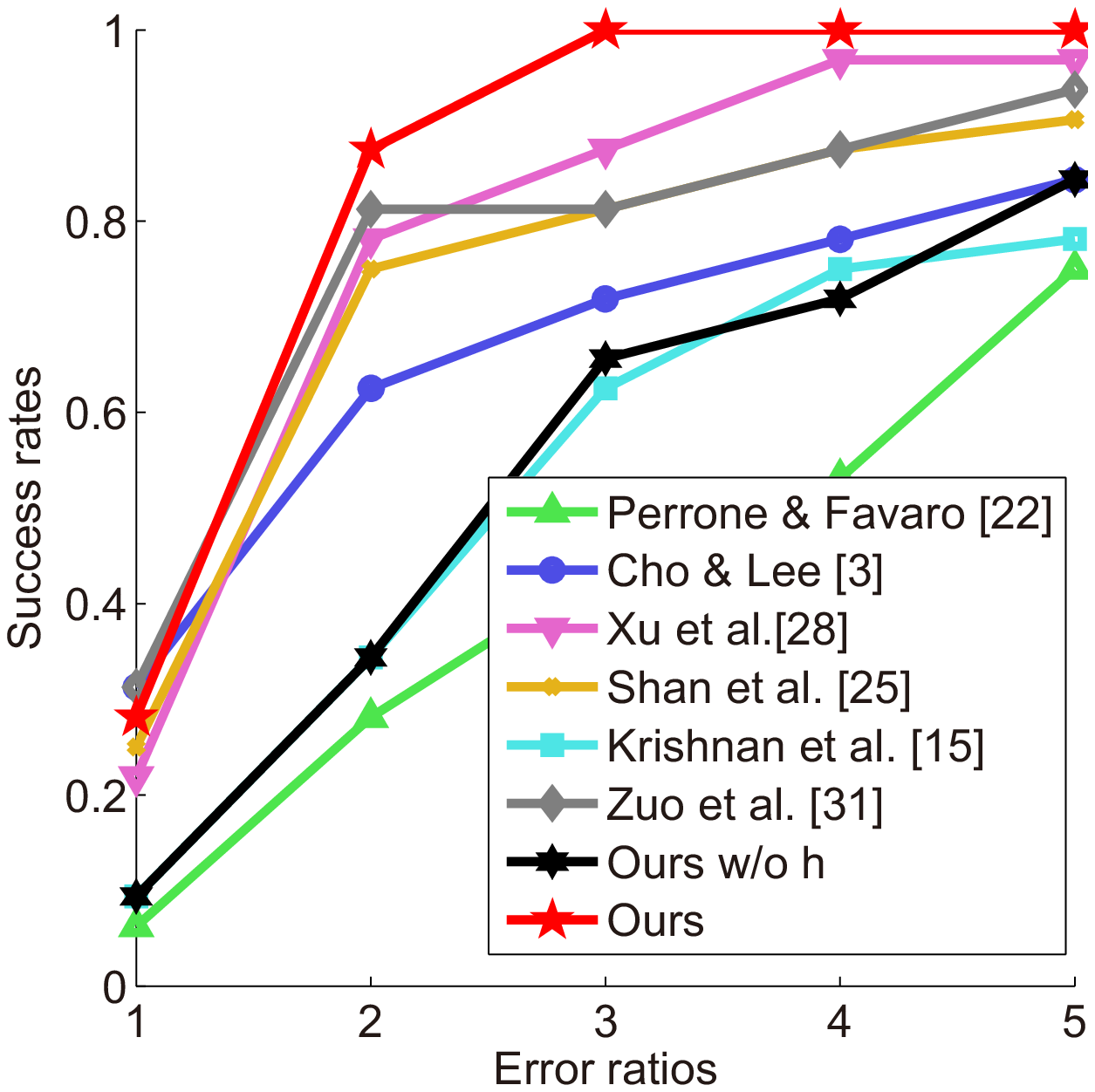}\\
  \small{(a)} & \small{(b)}\\ 
\end{tabular}\\
   \caption{Success rates on truth (a) and double (b) sizes. }
\label{fig10}
\end{figure}

\subsection{Effects of low rank-based regularization}
Corresponding to high error ratios of large kernels in Figure ~\ref{fig1}, 
% and to show the efficiency of our regularization,
we repeat the experiment using same parameters except $\mu$ and $\sigma$. Figure ~\ref{fig8} shows
low-rank regularized kernels are much more robust to kernel size. Noises in kernels are efficiently reduced and qualities of restored images are enhanced.
We further verify it on real-world images by imposing different regularization terms. As in Figure ~\ref{fig9}, blur kernels with low-rank regularization have less noises, while the others suffer from strong noises, yielding artifacts in the deblurring images. 
We note that in experiments of Figure ~\ref{fig8} and Figure ~\ref{fig9},  we deliberately omitted multi-scaling scheme to expose the effectiveness of low-rank regularization itself. 
%, because this operation may also reduce the larger-kernel effect by sub-sampling. Success rate under a specific error ratio value $err$ is defined as 
%$$\frac{\text{\# of samples $< err$ }}{\text{\# of all samples}} .$$

\subsection{Evaluation on synthetic dataset}
The proposed method is quantitatively evaluated on dataset from~\cite{levin2009understanding}. Figure ~\ref{fig10} shows the success rates of state-of-the-art methods versus our implementations with and without (set $\mu$ and $\sigma$ zero) low-rank regularization. The average PSNRs in Figure ~\ref{fig10} with different sizes are compared in Table~\ref{table:psnr}. Parameters are fixed during the whole experiment:  $\sigma=1$, $\mu=1$, $\tau=5\times10^{-5}$, $OuterIterMax=20$, $CGIterMax=3$ and $innerIterMax=10$; a 7-layer multi-scaling pyramid is taken. Kernel elements smaller than 1/20 of the maximum are cut to zero, which is also taken in \cite{cho2009fast,krishnan2009fast}. Low-rank regularization works more effectively than the regularization-free implementation and the state-of-art.

\begin{table}[t]
\small
\centering % used for centering table
\begin{tabular}{cccc} % centered columns (4 columns)
% \hline %inserts double horizontal lines
% \cite{shan2008high} & \cite{cho2009fast} & \cite{krishnan2011blind} & \cite{perrone2014total} & \cite{zuo2015discriminative} & $\sigma=0$ & $\sigma=1$ \vspace{2pt}\\  % inserts table
%heading

% \hline
%  % inserts single horizontal line
%   % [1ex]
  
%    26.54 & 26.85  & 25.34 & 27.34 & 26.58 & 26.68 & \textbf{27.36} \vspace{2pt}\\
%    \hline

%    26.44 & 25.74  & 23.95 &  23.29 & 26.83 & 23.85 & \textbf{27.47} \vspace{2pt}
%    \\
% \hline %inserts single line
\hline
Method & prior & truth size & double size
\\
\hline
\cite{perrone2014total} &-- &  27.34 &23.29
\\\cite{cho2009fast} & $\ell_2$  &26.85 & 25.74
\\ \cite{xu2010two} & $\ell_2$  &26.91 & 26.71
\\ \cite{shan2008high}   & $\ell_1$  & 26.54 & 26.44
\\\cite{krishnan2011blind} & $\ell_1$  &25.34 & 23.95
\\\cite{zuo2015discriminative} & $\ell_\alpha$  &26.58 & 26.83
\\\hline
 $\sigma=0$ & -- & 26.68 & 23.85
\\$\sigma=1$ &$\log\det$ &  \textbf{27.36} & \textbf{27.47} \\ 
\hline

\end{tabular} \vspace{10pt} \\
\caption{Average PSNRs (dB) with truth and double sizes in experiments of Figure ~\ref{fig10}.} % title of Table
\label{table:psnr} % is used to refer this table in the text
\end{table}

\begin{figure*}[t!]
\begin{center}
  \setlength{\tabcolsep}{1pt}
  % Requires \usepackage{graphicx}
  \small
  \begin{tabular}{cccc}
   \includegraphics[width=0.21\linewidth]{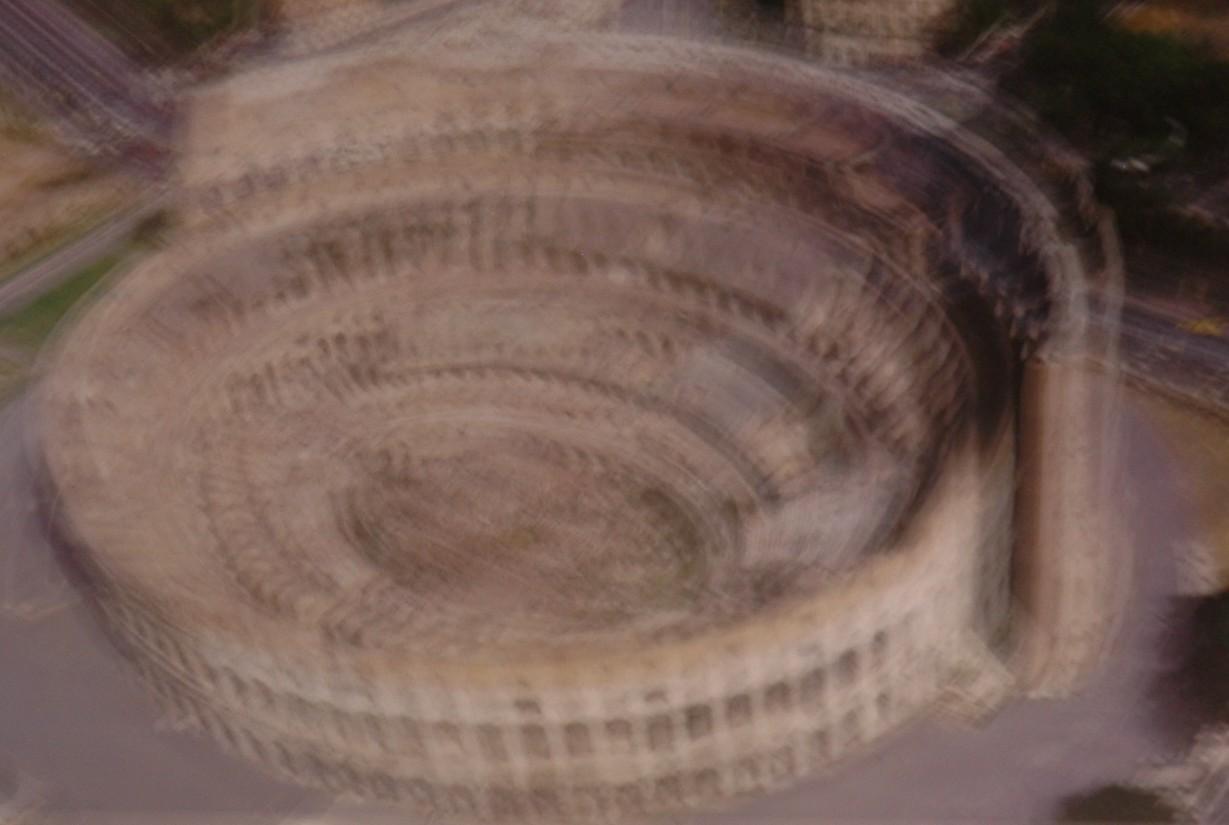}
  &\includegraphics[width=0.21\linewidth]{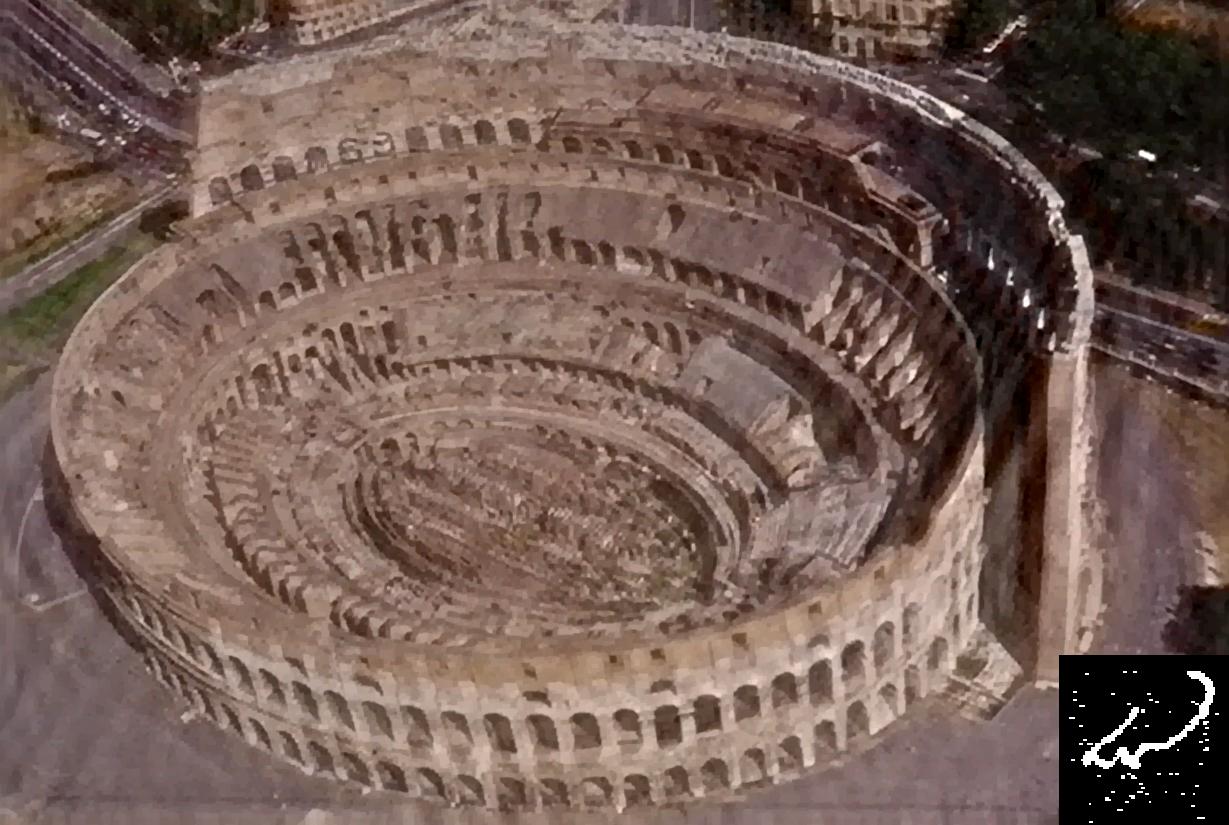}
  &\includegraphics[width=0.21\linewidth]{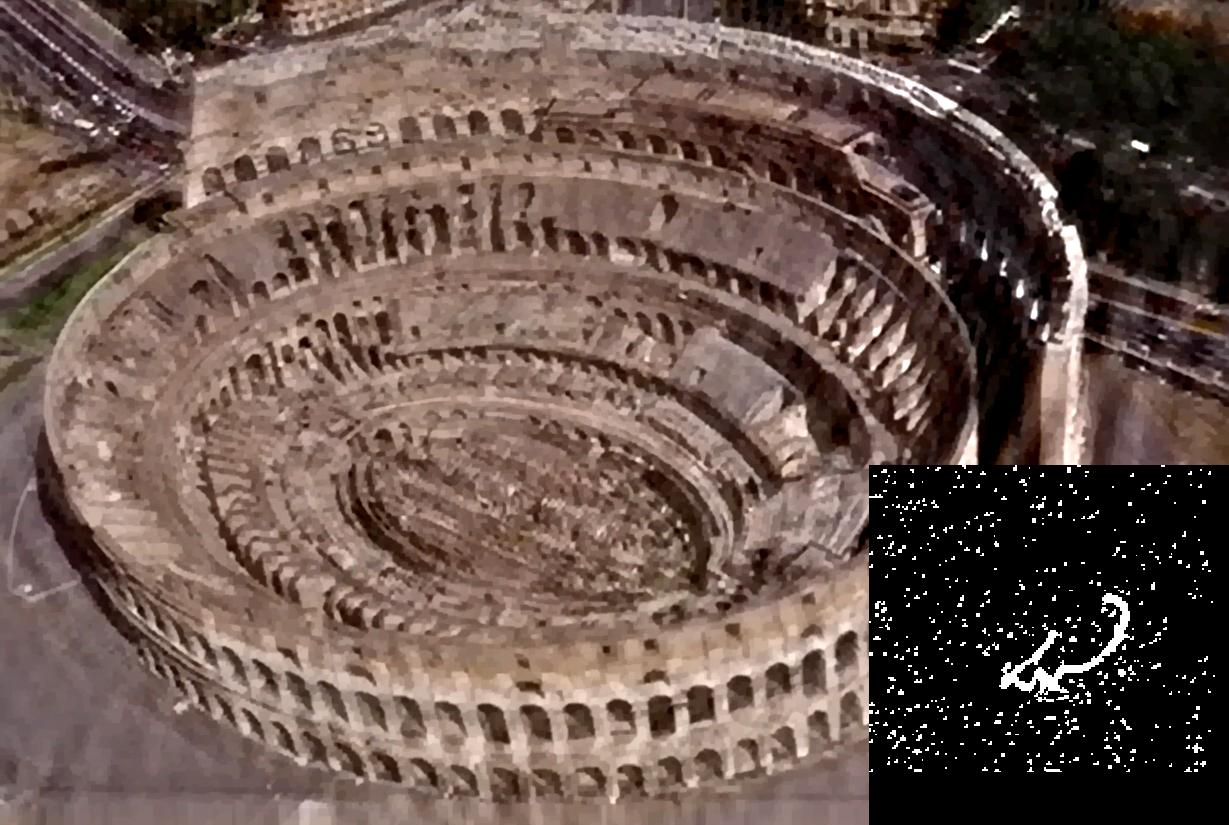} 
  &\includegraphics[width=0.21\linewidth]{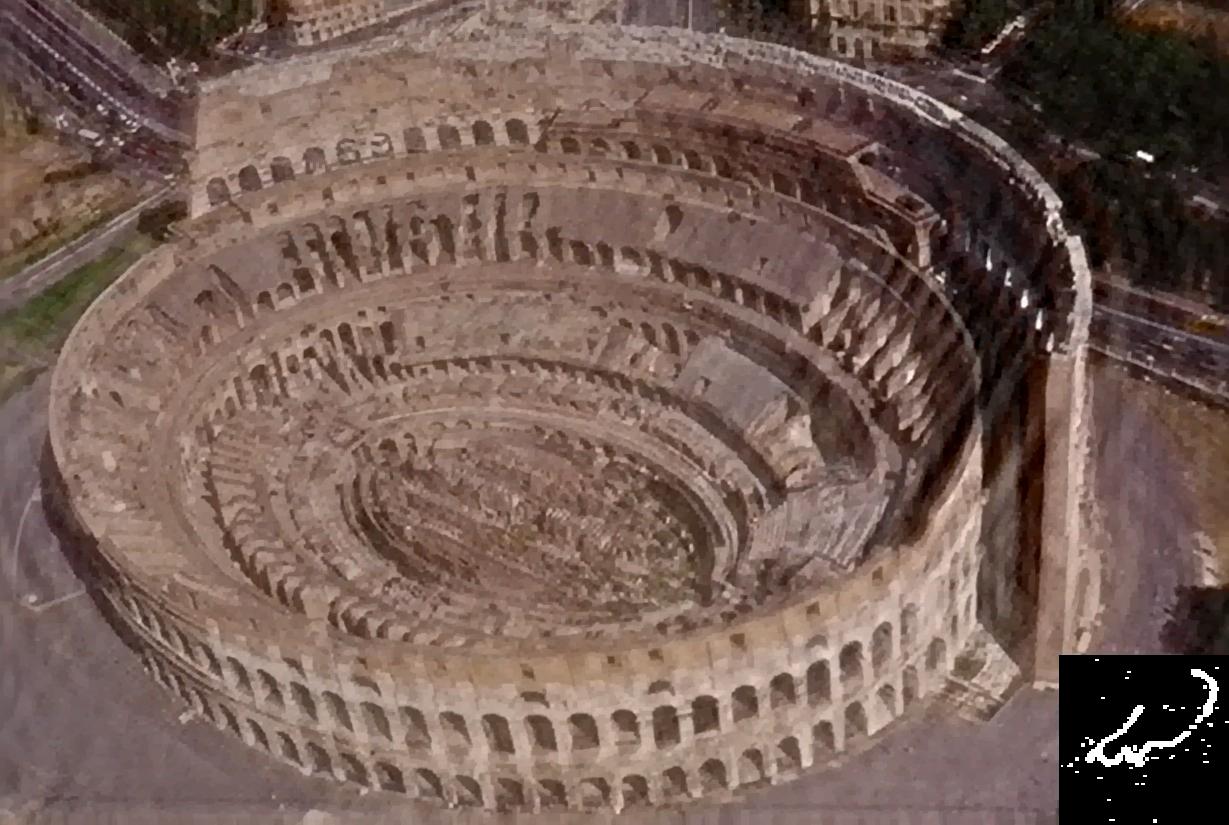}
  \vspace{-4pt}\\
  Blurry & $\ell_2$~\cite{cho2009fast} $85\times85$ & $\ell_2$~\cite{cho2009fast} $185\times185$ & $\ell_2$~\cite{xu2010two} $85\times85$\\
   \includegraphics[width=0.21\linewidth]{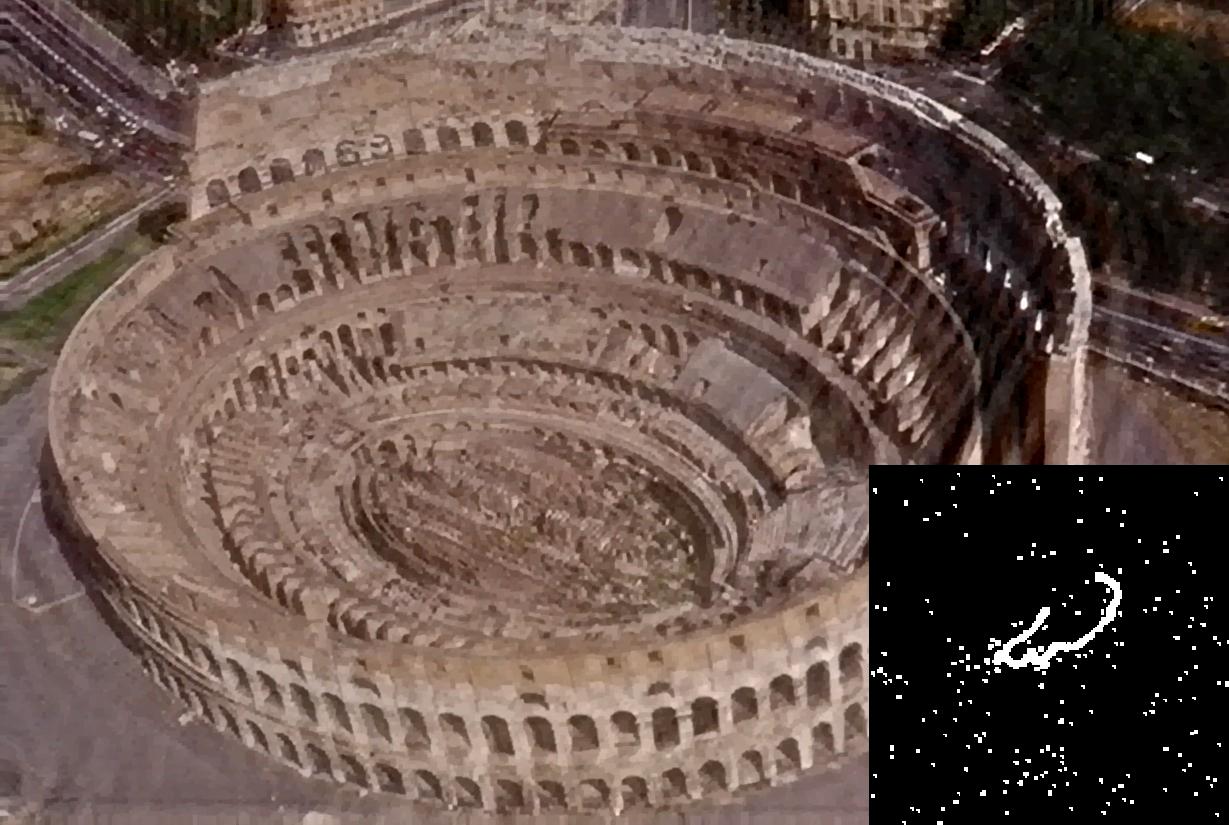}
  &\includegraphics[width=0.21\linewidth]{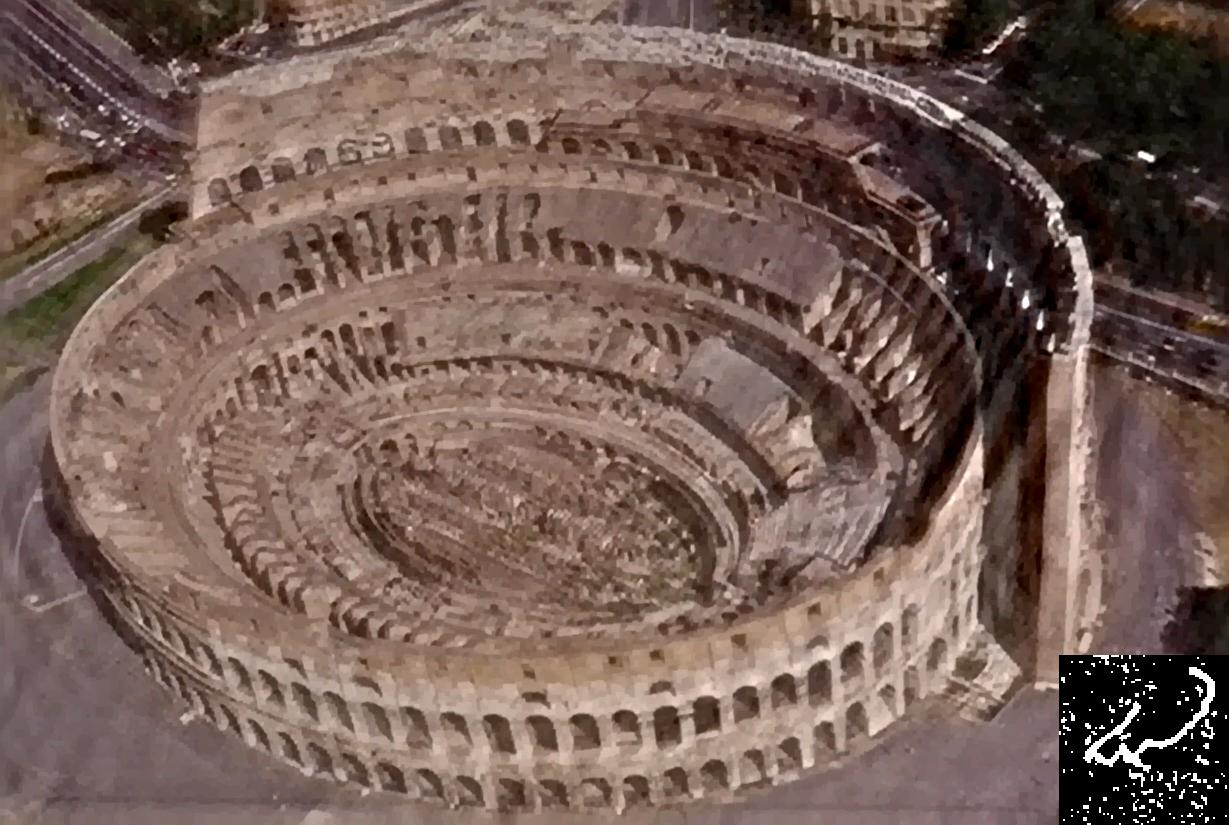}
  &\includegraphics[width=0.21\linewidth]{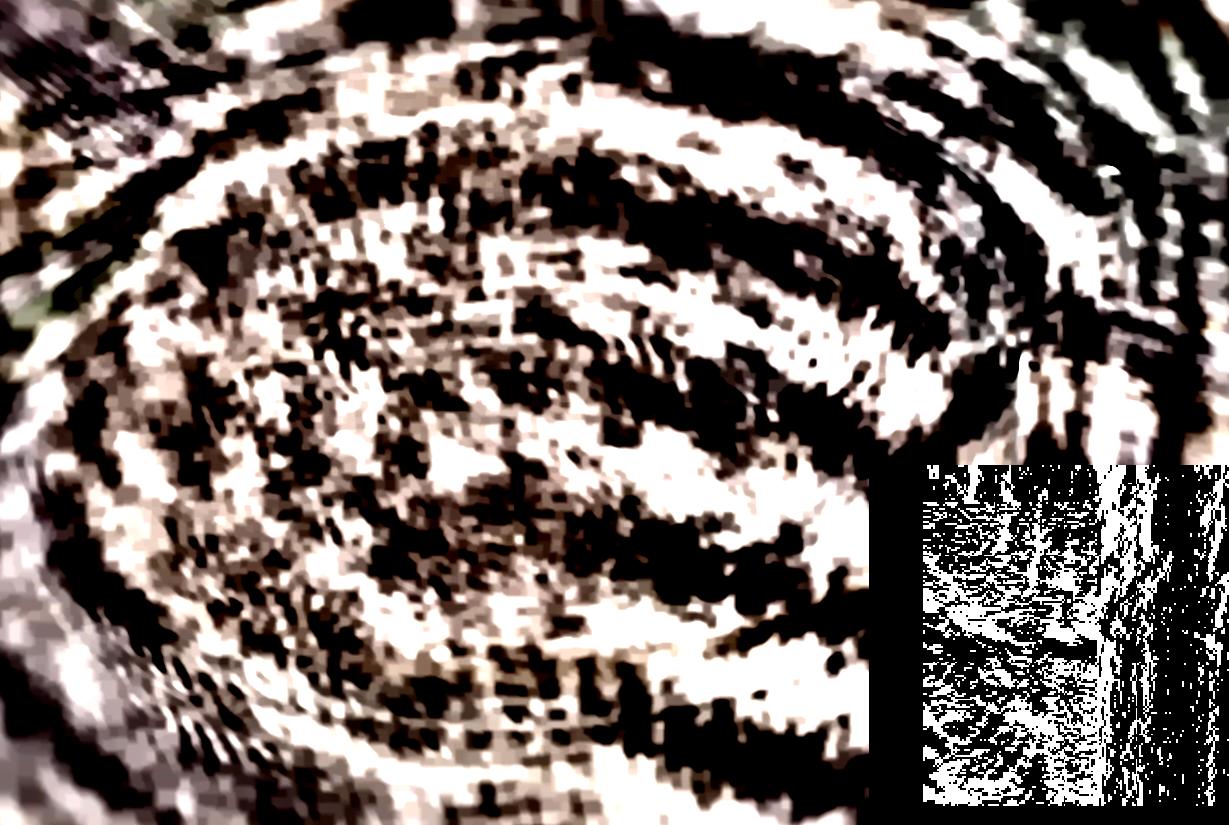}
  &\includegraphics[width=0.21\linewidth]{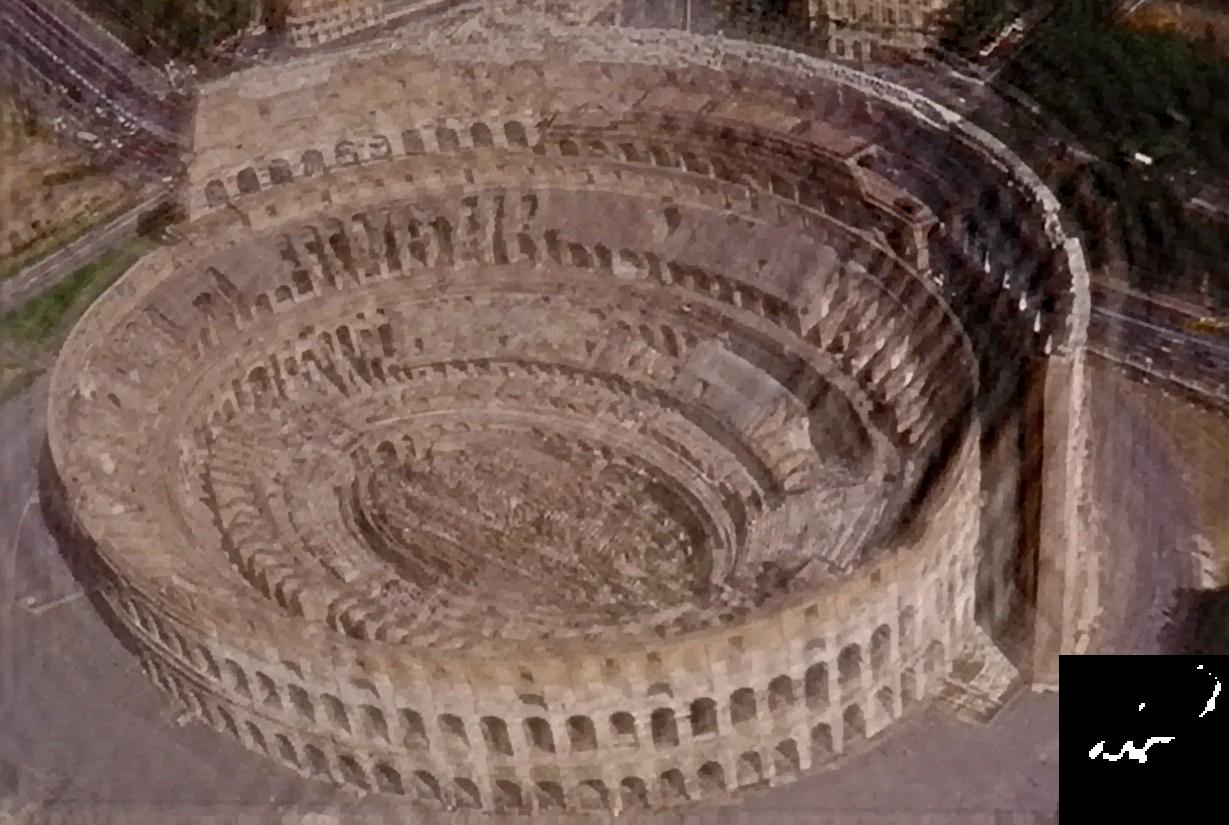} \vspace{-4pt}\\
  $\ell_2$~\cite{xu2010two} $185\time185$  & $\ell_1$~\cite{krishnan2011blind} $85\times85$ & $\ell_1$~\cite{krishnan2011blind} $185\time185$ & $\ell_\alpha$~\cite{zuo2015discriminative} $85\times85$ \\
   \includegraphics[width=0.21\linewidth]{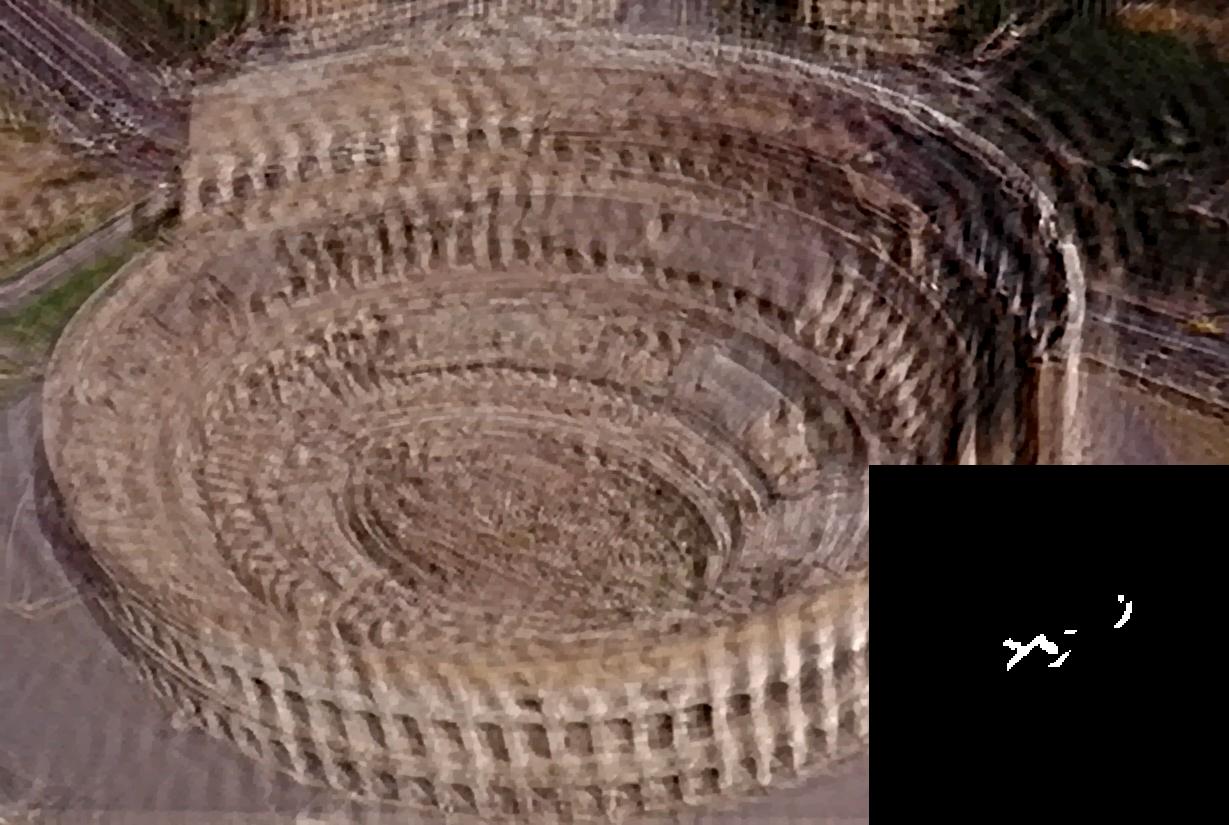}
  &\includegraphics[width=0.21\linewidth]{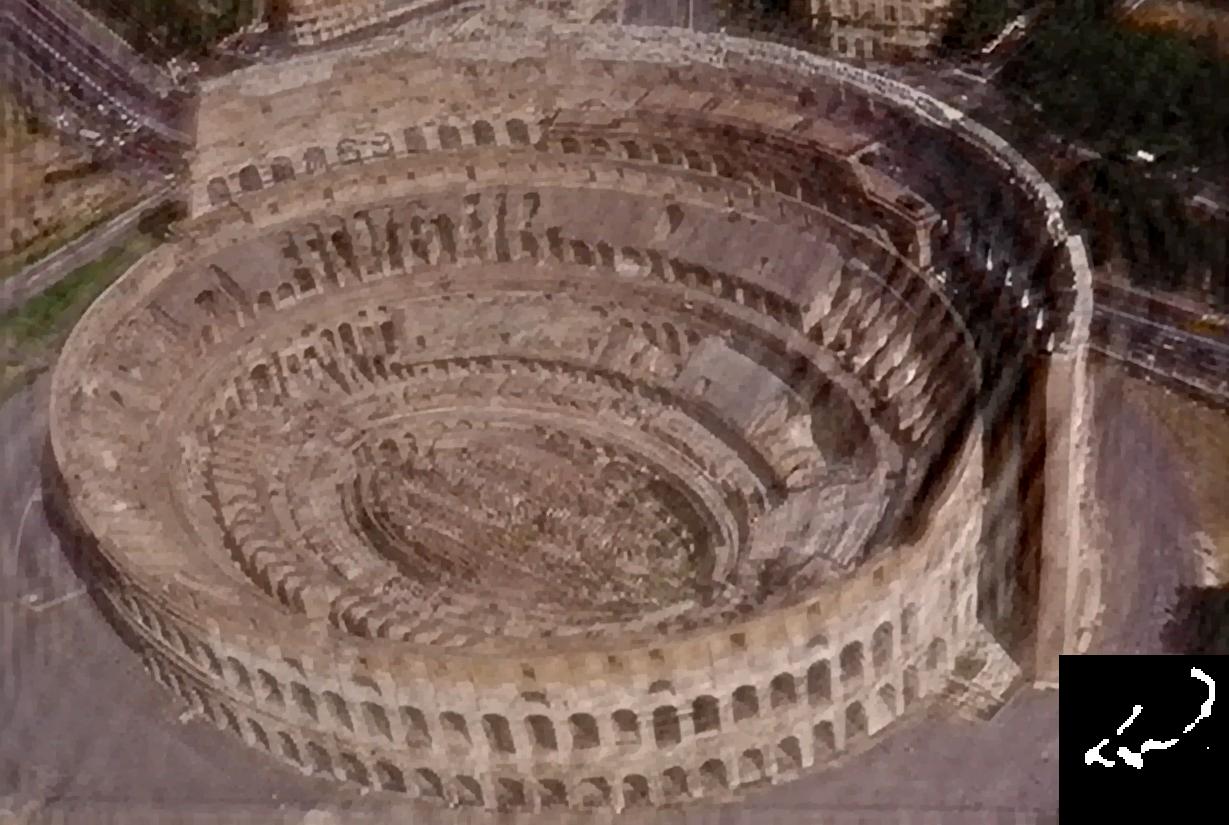}
  &\includegraphics[width=0.21\linewidth]{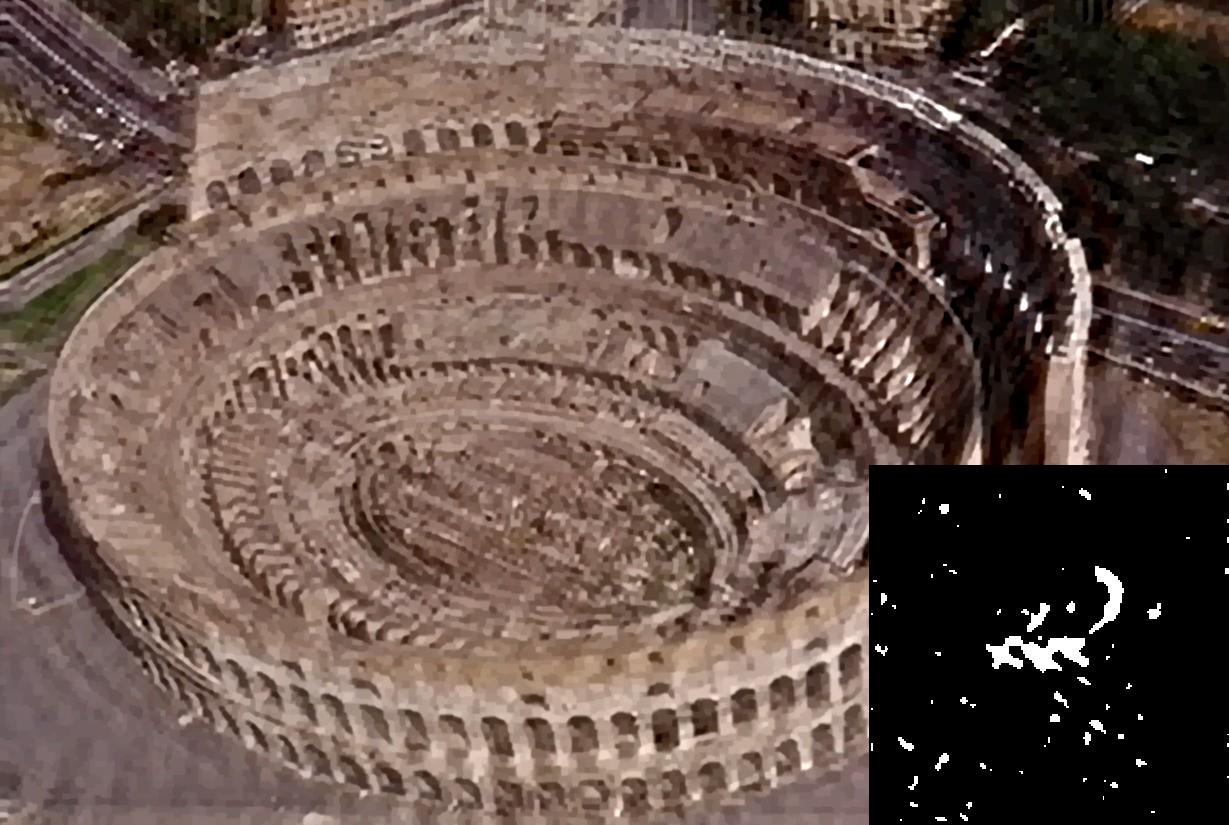}
  &\includegraphics[width=0.21\linewidth]{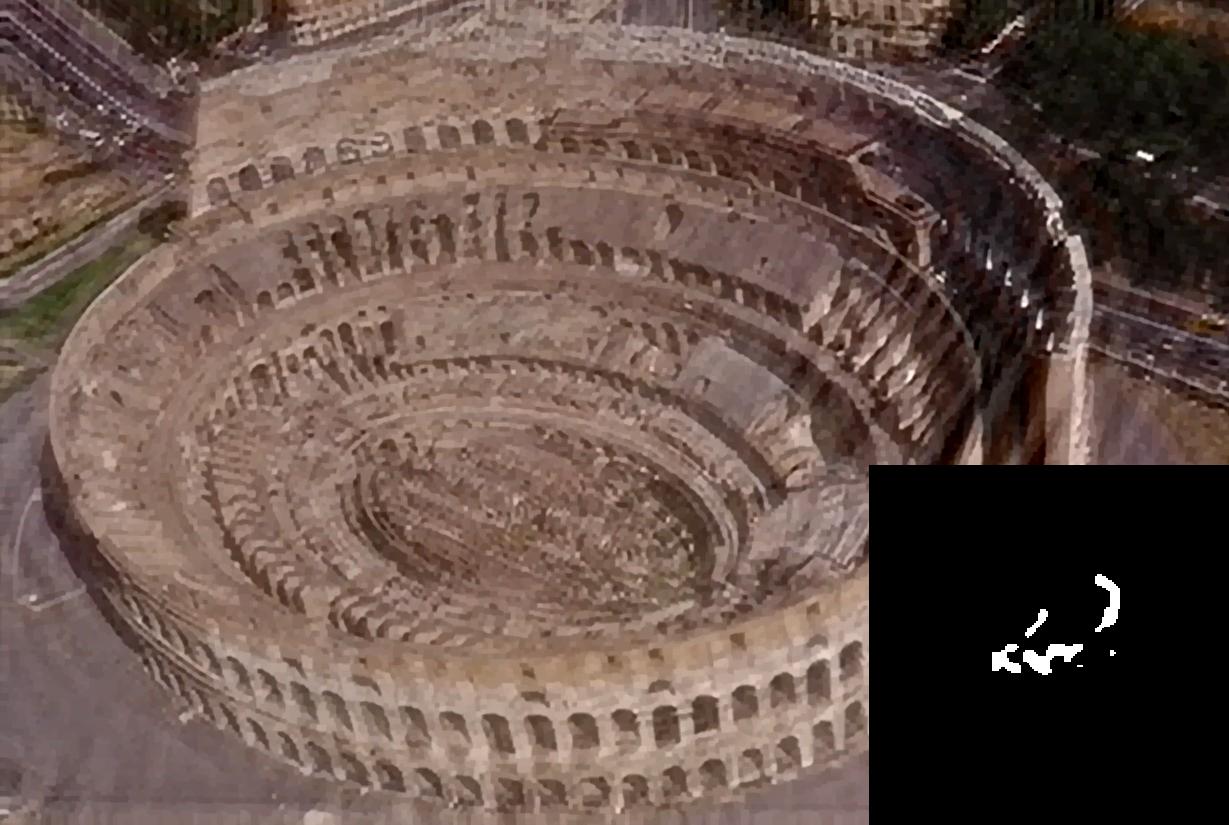} \vspace{-4pt}\\
  $\ell_\alpha$~\cite{zuo2015discriminative} $185\times185$ & Low rank (ours) $85\times85$ &None (ours) &Low rank (ours) $185\times185$\\
  \end{tabular}\\
\end{center}
   \caption{Test on real-world image \texttt{roma}. Each domain (positive parts) of estimated kernel is displayed at the bottom right corner of corresponding restored image.}
\label{fig11}
\end{figure*}

\begin{figure*}[t!]
\begin{center}
  \setlength{\tabcolsep}{2pt}
  % Requires \usepackage{graphicx}
  \small
  \begin{tabular}{ccc}
   \includegraphics[width=0.26\linewidth]{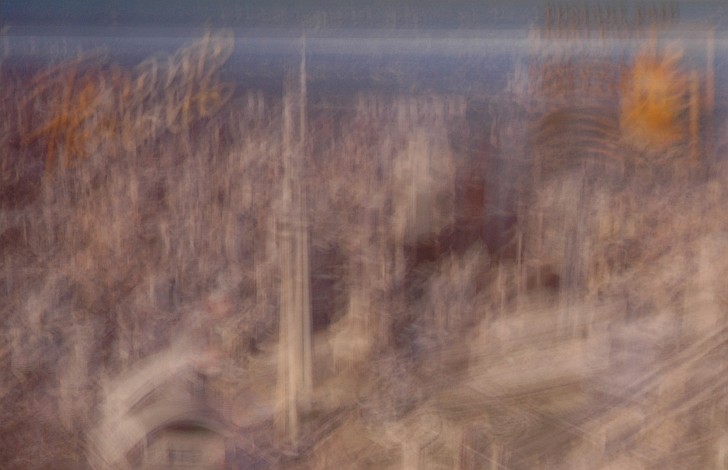}
  &\includegraphics[width=0.26\linewidth]{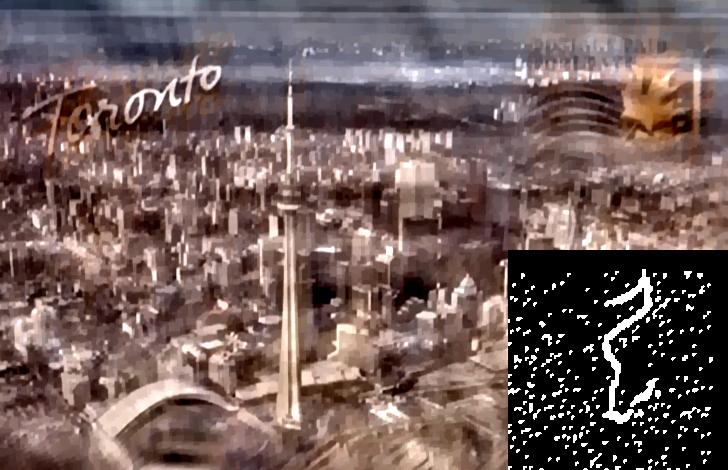}
  &\includegraphics[width=0.26\linewidth]{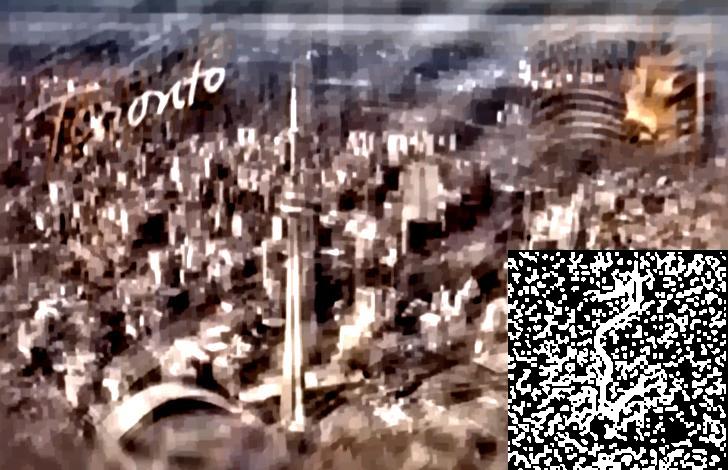}   \vspace{-2pt}\\
Blurry & $\ell_2$, 1/20 max threshold~\cite{cho2009fast} & $\ell_2$, heuristic domain detector~\cite{xu2010two} \\

   \includegraphics[width=0.26\linewidth]{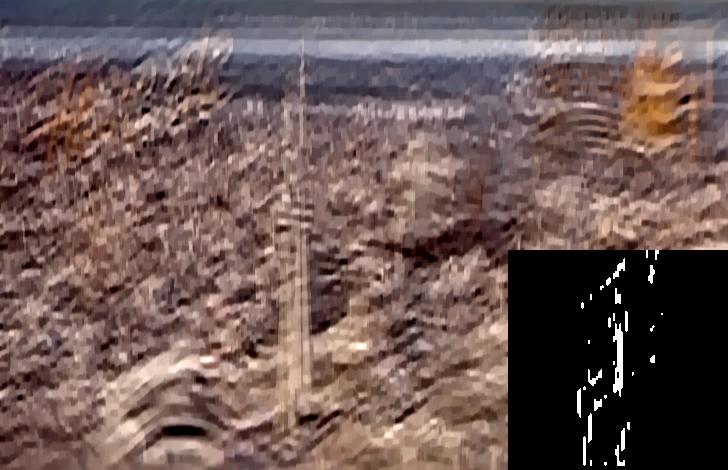}
  &\includegraphics[width=0.26\linewidth]{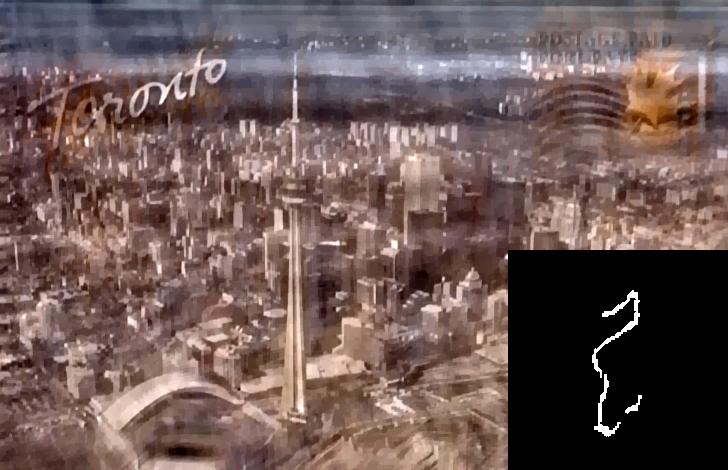}
  &\includegraphics[width=0.26\linewidth]{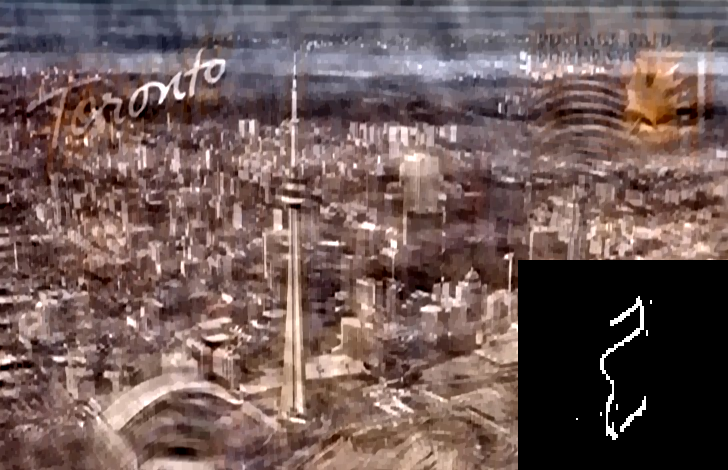}
  \vspace{-2pt}\\
  $\ell_1$~\cite{krishnan2011blind}, 1/20 max threshold & $\ell_\alpha$, none~\cite{zuo2015discriminative} & Low rank, 1/20 max threshold\\

  % Blurry & [3] & $\ell_2$~[21] & [13] &[23] & Low rank\\
  \end{tabular}\\
\end{center}
   \caption{Test on real-world image \texttt{postcard}. Kernel regularizations are listed under restored images.}
\label{fig12}
\end{figure*}

\subsection{Evaluation on real-world blurry images}
We compared our implementation to state-of-the-art methods on real-world images to reveal the robustness of low rank regularization on large kernel size. Specifically, \cite{xu2010two} takes a heuristic iterative support domain detector based on the differences of elements of $\hat k$, which is regarded to be more effective than 1/20 threshold. Figure ~\ref{fig11} shows that $185\times185$ size yields strong noises in estimated kernels of previous works~\cite{cho2009fast,xu2010two}, and even changes  main bodies of kernels~\cite{krishnan2011blind,zuo2015discriminative}. In contrast, low rank regularization can keep the kernel relatively stable for the larger size. One more comparison of different regularizations and refinement methods on large kernel size are shown in Figure~\ref{fig12}.  
As for computational efficiency of our method, it takes about 85s on a Lenovo ThinkCentre computer with Core i7 processor to process images with size  $255\times 255$. 

\section{Conclusion}
In this paper, we demonstrate that over-estimated kernel sizes produce increased noises in estimated kernel. We attribute the larger-kernel effect to the \textit{inflating effect}.  To reduce this effect, we propose a low-rank based regularization on kernel, which could suppress noise while remaining restored  main body of optimized kernel.

The success of blind deconvolution is contributed by many aspects. In practical implementations, even for noise-free $\mathbf y$, the intermediate $\mathbf{\hat x}^{(i)}$ is unlikely to iterate to ground truth, hence some parts of $\mathbf y$ will be treated as implicit noises, which may intensify the effect even more than expected and require future researches.

\section*{Acknowledgement}

This work is supported by the grants from the National Natural Science Foundation of China (61472043) and the National Key R\&D program of China (2017YFC1502505). 
We thank Ping Guo for constructive conversation.
% This paper is wished to be cited as Siyao~\etal .
Qian Yin is the corresponding author. 

\newpage
{\small
\bibliographystyle{ieee}
\bibliography{story}
}

\newpage

%%%%%%%%% BODY TEXT
\section*{Supplementary File: Proof to Theorem 2}
Assume $M$ to be odd and $M = 2m+1$ ($m\in\mathbb{N}^+$). Then,
 \begin{equation}\label{2}
\begin{aligned}
  &T_{\bm{x}}(M) \\
  =&\left[\begin{matrix}
x_{m+1}&   \cdots&  x_2&    x_1&    0&      \cdots& 0       \\
\vdots &    &       \vdots& x_2&    x_1&    \ddots& \vdots  \\
x_{M-1}&    &       \vdots& \vdots& x_2&    \ddots& 0       \\
x_M&        \ddots& \vdots& \vdots& \vdots& \ddots& x_1       \\
0&        \ddots& x_{M-1}& \vdots& \vdots& & x_2       \\
\vdots&       \ddots& x_M& x_{M-1}& \vdots& & \vdots       \\
0&        \cdots& 0& x_M& x_{M-1}&  \cdots& x_{m+1}
\end{matrix}\right]
\end{aligned}
\end{equation}

\begin{equation}
\begin{aligned}
=&\left[ J^{(-m)}\bm{x} \;  \cdots \;J^{(-1)}\bm{x} \;\; \bm{x} \;\; J^{(1)}\bm{x} \;
\cdots \; J^{(m)}\bm{x}\right].
\end{aligned}
\end{equation}

For any $M$-by-$M$ matrix A, $\mathrm{rank}(A)=M$ i.f.f. $\det(A)\neq0$. Thus,
\begin{equation}\label{pr_trans}
  \Pr\left(\mathrm{rank}\left(T_{\bm{x}}\left(M\right)\right)=M\right)=\Pr\left(\det\left(T_{\bm{x}}\left(M\right)\right)\neq0\right).
\end{equation}

As we know, the explicit formula of determinant of a Toeplitz matrix on its elements is unsolved in the current literature. Li~\cite{RefA} gives a concrete expression of $\det\left(T_{\bm{x}}\left(M\right)\right)$ by using LU factorization but fails to fit all situations (\eg~when $x_{m+1} = 0$). However, it can be shown that $\det\left(T_{\bm{x}}\left(M\right)\right)$ equals a multivariate polynomial function without manipulating the whole expression. By using Laplace expansion on $\det\left(T_{\bm{x}}\left(M\right)\right)$, the item of largest degree is $x^M_{m+1}$ with factor 1.

\newtheorem*{lemma*}{Lemma}
\begin{lemma*}
Let $X$ be a continuous r.v. in the finite support domain [a, b]. Let $P:\mathbb{R}\rightarrow\mathbb{R}$ be a polynomial function
\begin{equation*}\label{f_def}
  P(x)=x^k+Q(x)
\end{equation*}
where $Q$ is a finite polynomial function with the largest degree less than $k$. Generate a new r.v.
\begin{equation*}\label{Z_def}
  Y=P(X).
\end{equation*}
Then, for $\forall {y\in\mathbb{R}}$, the Cumulative Distribution Function (CDF) $F_Y$ is continuous at y.
\end{lemma*}
\begin{proof}
\begin{equation*}\label{lproof}
F_Y(y) = \Pr(Y<=y)=\int_{\Xi(y)}f_X(x)dx
\end{equation*}
where $\Xi(y)=\left\{x|P(x)<=y,  x\in\mathbb{R}\right\}$.

For $\forall{y\in\mathbb{R}}$,
\begin{equation*}
  \Xi(y^+) = \Xi(y)
\end{equation*}
and
\begin{equation*}
  \Xi(y^-) = \Xi(y)\setminus\Omega(y)
\end{equation*}
where $\Omega(y)=\left\{x|P(x)-y=0\right\}$.

Based on Beppo Levi's Theorem,
\begin{equation*}
  \lim_{\xi\rightarrow y^+}\int_{\Xi(\xi)}f_X(x)dx = \int_{\Xi(y)}f_X(x)dx.
\end{equation*}
Because $P(x)\not\equiv c$ ($c$ is a constant), for $\forall {y \in \mathbb{R}}$, zeros of $P(x)-y$ are finite, hence the Lebesgue measure of $\Omega(y)$ is zero. We have
\begin{equation*}
  \lim_{\xi\rightarrow y^-}\int_{\Xi(\xi)}f_X(x)dx = \int_{\Xi(y)}f_X(x)dx.
\end{equation*}
Thus
\begin{equation*}
  F_Y(y^+) = F_Y(y^-)=F(y).
\end{equation*}
\end{proof}

\newtheorem*{thm*}{Theorem 2}
\begin{thm*}
Let $X$ be a continuous r.v. with PDF
\begin{equation*}
  f_X(x) = \begin{cases}
\beta \exp\left(-\gamma |x|^{\alpha}\right) & x \in [-1, 1] \\
0 & otherwise.
\end{cases}
\end{equation*}
For a sample of independent observations $X_1,\ldots, X_M$, generate a new r.v.
\begin{equation*}
\begin{aligned}
  &Z= \\
  &\det\left[\begin{matrix}
X_{m+1}&   \cdots&  X_2&    X_1&    0&      \cdots& 0       \\
\vdots &    &       \vdots& X_2&    X_1&    \ddots& \vdots  \\
X_{M-1}&    &       \vdots& \vdots& X_2&    \ddots& 0       \\
X_M&        \ddots& \vdots& \vdots& \vdots& \ddots& X_1       \\
0&        \ddots& X_{M-1}& \vdots& \vdots& & X_2       \\
\vdots&       \ddots& X_M& X_{M-1}& \vdots& & \vdots       \\
0&        \cdots& 0& X_M& X_{M-1}&  \cdots& X_{m+1}
\end{matrix}\right].
\end{aligned}
\end{equation*}
Then,
\begin{equation*}
  \Pr\left(Z=0\right) = 0.
\end{equation*}
\end{thm*}
\begin{proof}
Based on the Law of Total Probability and Dominated Convergence Theorem,
\begin{equation*}
\begin{aligned}
   &&&\Pr\left(Z=0\right) \\
 &=&&F_Z(0)-F_Z(0^-)\\
   &=&& \int\displaylimits_{-\infty}^{\infty} \dotsi \int\displaylimits_{-\infty}^{\infty} \int\displaylimits_{-\infty}^\infty \dotsi \int\displaylimits_{-\infty}^\infty F_{Y\left(
   \xi_1, \ldots, \xi_m, \xi_{m+2}, \ldots, \xi_{M}\right)}(0) \\
   &&&f_X(\xi_1) \dotsm f_X(\xi_m) f_X(\xi_{m + 2}) \dotsm f_X(\xi_M)\\
   &&&\mathrm{d} \xi_1 \dotsm \mathrm{d} \xi_m \, \mathrm{d} \xi_{m + 2} \dotsm \mathrm{d} \xi_M\\
   &-&&\lim_{z\rightarrow0^-}\int\displaylimits_{-\infty}^\infty \dotsi \int\displaylimits_{-\infty}^\infty \int\displaylimits_{-\infty}^\infty \dotsi \int\displaylimits_{-\infty}^\infty F_{Y\left(
   \xi_1, \ldots, \xi_m, \xi_{m+2}, \ldots, \xi_{M}\right)}(z) \\
   &&&f_X(\xi_1) \dotsm f_X(\xi_m) f_X(\xi_{m + 2}) \dotsm f_X(\xi_M)\\
   &&&\mathrm{d} \xi_1 \dotsm \mathrm{d} \xi_m \, \mathrm{d} \xi_{m + 2} \dotsm \mathrm{d} \xi_M
   \end{aligned}
\end{equation*}
\begin{equation*}
\begin{aligned}
   &=&&\int\displaylimits_{-\infty}^\infty \dotsi \int\displaylimits_{-\infty}^\infty \int\displaylimits_{-\infty}^\infty \dotsi \int\displaylimits_{-\infty}^\infty \big( F_{Y\left(
   \xi_1, \ldots, \xi_m, \xi_{m+2}, \ldots, \xi_{M}\right)}(0)\\
   &&&-F_{Y\left(
   \xi_1, \ldots, \xi_m, \xi_{m+2}, \ldots, \xi_{M}\right)}(0^-)\big) f_X(\xi_1) \dotsm f_X(\xi_m)\\
   &&& f_X(\xi_{m + 2}) \dotsm f_X(\xi_M)\mathrm{d} \xi_1 \dotsm \mathrm{d} \xi_m \, \mathrm{d} \xi_{m + 2} \dotsm \mathrm{d} \xi_M
\end{aligned}
\end{equation*}
where
\begin{equation*}
\begin{aligned}
  &&&Y\left(
   \xi_1, \ldots, \xi_m, \xi_{m+2}, \xi_{M}\right) \\
  &&=&\det\left[\begin{matrix}
X_{m+1}&   \cdots&  \xi_2&    \xi_1&    0&      \cdots& 0       \\
\vdots &    &       \vdots& \xi_2&    \xi_1&    \ddots& \vdots  \\
\xi_{M-1}&    &       \vdots& \vdots& \xi_2&    \ddots& 0       \\
\xi_M&        \ddots& \vdots& \vdots& \vdots& \ddots& \xi_1       \\
0&        \ddots& \xi_{M-1}& \vdots& \vdots& & \xi_2       \\
\vdots&       \ddots& \xi_M& \xi_{M-1}& \vdots& & \vdots       \\
0&        \cdots& 0& \xi_M& \xi_{M-1}&  \cdots& X_{m+1}
\end{matrix}\right]\\
\\
&&=&X_{m+1}^M+ Q_{\xi_1, \,\dots, \,\xi_m, \,\xi_{m+2}, \,\ldots, \,\xi_M}\left(X_{m+1}\right).
\end{aligned}
\end{equation*}

$Q_{\xi_1, \,\dots, \,\xi_m, \,\xi_{m+2}, \,\ldots, \,\xi_M}$ is a polynomial function with the largest degree less than $M$. Based on Lemma, we have
\begin{equation*}
  F_{Y\left(
   \xi_1, \ldots, \xi_m, \xi_{m+2}, \ldots, \xi_{M}\right)}(0)-F_{Y\left(\xi_1, \ldots, \xi_m, \xi_{m+2}, \ldots, \xi_{M}\right)}(0^-) = 0.
\end{equation*}

Hence,
\begin{equation*}
  \Pr(Z=0).
\end{equation*}
\end{proof}

\end{document}